%% file: main.tex
\documentclass[letterpaper,journal]{IEEEtran}
\usepackage{amsmath,amsfonts}
\usepackage{algorithmic}
\usepackage{algorithm}
\usepackage{array}
\usepackage[caption=false,font=normalsize,labelfont=sf,textfont=sf]{subfig}
\usepackage{textcomp}
\usepackage{stfloats}
\usepackage{url}
\usepackage{verbatim}
\usepackage{graphicx}
\usepackage{cite}
\usepackage{xspace}
\usepackage{multirow}
\usepackage{mathtools}
\usepackage{tikz}
\usepackage[dvipsnames, x11names]{xcolor}
\usepackage[hidelinks]{hyperref}
\usepackage[acronym]{glossaries-extra}
\usepackage{cleveref}
\usepackage{booktabs, nicematrix}
\usepackage{adjustbox}
\usepackage{circledsteps}
\usepackage{tikz}
\usepackage{tikz-cd}
\usetikzlibrary{shapes.geometric}
\usepackage{ifthen}

\usepackage{amsthm}
\theoremstyle{definition}
\newtheorem{definition}{Definition}
\newtheorem{remark}{Remark}

\newtheorem{theorem}{Theorem}
\newtheorem{lemma}[theorem]{Lemma}

\newtheorem{assumption}{Assumption}

\crefname{assumption}{Assumption}{Assumptions}
\Crefname{assumption}{Assumption}{Assumptions}
\crefname{definition}{Definition}{Definitions}
\Crefname{definition}{Definition}{Definitions}

\input{preamble}

\newboolean{anonymous}
\setboolean{anonymous}{false}
\newboolean{revisions}
\setboolean{revisions}{false}

\begin{document}
\bstctlcite{BSTcontrol}

\title{Task-Driven Co-Design of Heterogeneous Multi-Robot Systems}

\ifthenelse{\boolean{anonymous}}{%
  \author{Anonymous Authors}%
}{%
\author{Maximilian Stralz, Meshal Alharbi, Yujun Huang, Gioele Zardini
\thanks{The authors are with the Laboratory for Information \& Decision Systems, Massachusetts Institute of Technology, Cambridge, MA 02139 {\tt \{mstralz,meshal,yujun233,gzardini\}@mit.edu}}
\thanks{This material is based on research supported in part by the Defense Advanced Research Projects Agency (DARPA) under Award D25AC00373, and in part by the Air Force Office of Scientific Research (AFOSR) under Grant Agreement FA9550261B038.
The views and conclusions contained in this document are those of the authors and should not be interpreted as representing the official policies, either expressed or implied, of the U.S. Government.}
}
}

\maketitle
\input{Sections/0_abstract}


\input{Sections/1_introduction}
\input{Sections/2_related_work}
\input{Sections/3_mathematical_preliminaries}
\input{Sections/4_problem_formulation}
\input{Sections/5_formal_codesign_model}
\input{Sections/6_case_studies}
\input{Sections/7_conclusion_future_work}

\bibliographystyle{IEEEtran}
\bibliography{references}

\clearpage

\end{document}

%% file: preamble.tex
\DeclareRobustCommand{\rev}[1]{%
  \ifthenelse{\boolean{revisions}}{%
    \ifmmode \begingroup\color{blue}#1\endgroup
    \else \textcolor{blue}{#1}%
    \fi
  }{#1}%
}
\newcommand{\defeq}{\mathrel{\raisebox{-0.3ex}{$\overset{\text{\tiny def}}{=}$}}}
\newcommand{\setWithArg}[2]{\{#1 \mid #2\}}

\newcommand{\setproduct}{\times}

\newcommand{\mapFollowing}{\circ}

\newcommand{\op}{^{\mathrm{op}}}
\newcommand{\inv}{^{-1}}

\newcommand{\SetA}{A}
\newcommand{\SetB}{B}
\newcommand{\SetC}{C}

\newcommand{\setelmA}{m_\SetA}
\newcommand{\setelmB}{m_\SetB}

\newcommand{\mapf}{f}
\newcommand{\mapg}{g}
\newcommand{\maph}{h}

\newcommand{\subsetYB}{Y_\SetB}

\newcommand{\posetP}{P}
\newcommand{\posetQ}{Q}

\newcommand{\tup}[1]{\left( #1 \right)}

\newcommand{\posetleq}{\preceq}

\newcommand{\poselxP}{x_\posetP}
\newcommand{\poselyP}{y_\posetP}
\newcommand{\poselxQ}{x_\posetQ}
\newcommand{\poselyQ}{y_\posetQ}

\newcommand{\powerset}[1]{\mathtt{Pow}(#1)}

\setabbreviationstyle[acronym]{long-short}
\glssetcategoryattribute{acronym}{nohyperfirst}{true}

\newacronym{abk:poset}{poset}{partially ordered set}
\newacronym{acr:mapf}{MAPF}{Multi-Agent Path Finding}
\newacronym{acr:cpp}{CPP}{coverage path planning}
\newacronym{acr:sar}{SAR}{search-and-rescue}
\newacronym{acr:mrta}{MRTA}{multi-robot task allocation}
\newacronym{acr:stc}{STC}{spanning tree coverage}
\newacronym{acr:agd}{AGD}{adaptive grid-based decomposition}
\newacronym{acr:av}{AVs}{autonomous vehicles}
\newacronym{acr:mip}{MIP}{mixed-integer programming}
\newacronym{acr:tsp}{TSP}{traveling salesman problem}
\newacronym{acr:darp}{DARP}{Divide Areas Algorithm for Optimal Multi-Robot Coverage Path Planning}
\newacronym{acr:uav}{UAV}{Unmanned Aerial Vehicle}
\newacronym[longplural={monotone design problems with implementation}]{acr:mdpi}{MDPI}{monotone design problem with implementation}
\newacronym{acr:mas}{MAS}{multi-agent systems}
\newacronym{acr:hv}{HV}{hypervolume}
\newacronym{acr:gdp}{GD+}{generational distance plus}
\newacronym{acr:igdp}{IGD+}{inverted generational distance plus} 
\newcommand{\NSGAII}{\textsc{NSGA-II}\xspace}
\newcommand{\MOEAD}{\textsc{MOEA/D}\xspace}

\definecolor{dpred}{rgb}{0.7, 0.0, 0.0}
\definecolor{dpgreen}{rgb}{0.0, 0.5, 0.0}

\newcommand{\F}[1]{{\color{dpgreen}#1}}
\newcommand{\R}[1]{{\color{dpred}#1}}

\def\prov{\mathsf{prov}}
\def\req{\mathsf{req}}

\newcommand{\setOfFunctionalitiesOp}[1]{\F{\mathcal{F}_{#1}}^{\mathrm{op}}}
\newcommand{\setOfFunctionalities}[1]{\F{\mathcal{F}_{#1}}}
\newcommand{\setOfResources}[1]{\R{\mathcal{R}_{#1}}}
\newcommand{\setOfImplementations}[1]{\mathcal{I}_{#1}}

\DeclareRobustCommand{\circled}[1]{%
  \tikz[baseline=(char.base)]{
    \node[circle,draw,inner sep=0.5pt] (char) {#1};
  }%
}
\newcommand{\Circle}[1]{{\footnotesize\circled{#1}}}
\newcommand{\map}{M}
\newcommand{\setofmaps}{\mathcal{M}}
\newcommand{\robot}{R}

\newcommand{\dprop}{\mathcal{D}}

\newcommand{\sensing}{\mathcal{S}}

\newcommand{\fleet}{\mathcal{F}}
\newcommand{\planner}{\operatorname{planner}}
\newcommand{\wpoints}{W}
\newcommand{\collwpoints}{\mathbf{W}}
\newcommand{\setcollwpoints}{\mathcal{W}}
\newcommand{\traj}{J}
\newcommand{\colltraj}{\mathbf{J}}
\newcommand{\setcolltraj}{\mathcal{T}}
\newcommand{\tbud}{\tau}
\newcommand{\tbudvec}{\boldsymbol{\tau}}
\newcommand{\exec}{\operatorname{executor}}
\newcommand{\res}{\operatorname{require}}
\newcommand{\resrobot}{\mathcal{C}}
\newcommand{\perfspace}{\mathcal{P}}
\newcommand{\perfvecspace}{\boldsymbol{P}}
\newcommand{\resposet}{\mathcal{C}}
\newcommand{\eval}{\operatorname{evaluator}}
\newcommand{\task}{\operatorname{task}}
\newcommand{\missiontime}{t_{\mathrm{mc}}}

%% file: Sections/0_abstract.tex
\begin{abstract}
The design of multi-agent robotic systems \rev{involves} tightly coupled decisions spanning heterogeneous domains, including robot design, fleet composition, and planning.
Much effort has been devoted to isolated improvements in these domains, while system-level co-design considering trade-offs and task requirements remains underexplored.
In this work, we present a formal and compositional framework for the task-driven co-design of heterogeneous multi-robot systems\rev{, and propose robotic phase diagrams as a qualitative design guideline.}
Building on \rev{monotone co-design theory}, we introduce general abstractions of robots, fleets, planners, executors, and evaluators as interconnected design problems with well-defined interfaces that are agnostic to both implementations and tasks.
This structure enables efficient joint optimization of robot design, fleet composition, and planning under task-specific performance constraints.
A series of case studies demonstrates the capabilities of the framework.
\rev{New component solutions can be seamlessly incorporated, including robot types}, task profiles, and probabilistic sensing objectives, while non-obvious design alternatives are systematically uncovered with optimality guarantees.
The results highlight the flexibility, scalability, and interpretability of the proposed approach, \rev{illustrate how formal co-design enables principled reasoning about complex heterogeneous multi-robot systems, and how a robotic phase diagram qualitatively guides the design.}
\end{abstract}

%% file: Sections/1_introduction.tex
\section{Introduction}
\label{sec:introduction}
\IEEEPARstart{M}{ulti-agent} robotic systems are increasingly deployed in applications where the workload is spatially distributed, time-critical, or inherently parallel, including warehouse logistics~\cite{wurman2008coordinating, standley2010finding}, domestic service~\cite{galceran2013survey}, and safety-critical inspection and search tasks such as gas-leak detection~\cite{murvay2012survey} and de-mining~\cite{galceran2013survey}.
In these settings, the system designer must decide not only \emph{how to plan}, but also \emph{what to build and deploy}: the number of robots, their types (e.g., ground vs. aerial), and the sensing/actuation/energy modules that determine each robot's capabilities.
Crucially, these choices are tightly coupled.
Fleet composition shapes what planning algorithms are feasible and effective; robot-level design choices (battery sizing, actuation, sensors) determine execution horizons and sensing performance; and the task profile and environment determine which couplings matter.
As a result, system design is rarely reducible to optimizing a single component in isolation~\cite{saberifar2022charting, zardini2023dissertation, mcfassel2023reactivity}.

A recurring challenge is that the relevant design space is both \emph{heterogeneous} and \emph{combinatorial}: discrete module choices and planner selections interact with continuous quantities such as time, energy, and performance.
This makes exhaustive evaluation by simulation or prototyping impractical, and it makes it difficult to answer even basic questions such as: 
\emph{When does a heterogeneous fleet outperform a homogeneous one?}
\emph{When is it worth switching planning software rather than upgrading hardware?}
\emph{How do optimal designs change as task requirements or environments change?}

\input{Figures/overview_diagram}

\subsection{Robotics Phase Diagrams}
\rev{These questions motivate a ``phase-diagram'' perspective on multi-agent system design.
In physics, phase diagrams partition external conditions into regimes with qualitatively different optimal structures (e.g., ice/water/vapor).
Analogously, we partition the space of \emph{task and environment parameters} into regimes where qualitatively different \emph{system architectures} are Pareto-optimal, separated by boundaries at which it becomes optimal to increase fleet size, introduce a new robot type, or switch planning strategy.
Such \emph{multi-agent robotics phase diagrams} require a representation in which fleet design, planning, execution, and task evaluation can be varied systematically while remaining interpretable and computationally tractable.}
\rev{This paper presents such a formal representation, together with phase diagrams of a specific multi-agent task (\Cref{fig:overview}).}

\subsection{\rev{Approach and Contributions}}
\rev{We develop a formal, compositional framework for heterogeneous multi-robot system design based on the monotone theory of co-design~\cite{censi2015mathematical,zardini2023dissertation}.
Our contributions are threefold.
First, we introduce implementation-agnostic abstractions of robots, fleets, planners that generate waypoints, executors that produce dynamically feasible trajectories under time/energy budgets, and evaluators that score trajectories through task-relevant metrics.
Each abstraction is translated into a monotone \emph{design problem} with explicit interfaces: \F{functionalities} (what the component provides) and \R{resources} (what it requires).
Importantly, because the interfaces are modular, new robot catalogs, planners, and task metrics can be integrated without changing the surrounding structure.
Second, we interconnect these design problems yielding a system-level co-design problem that can be queried to compute task-feasible, Pareto-optimal multi-agent system designs.
Finally, we instantiate the framework in simulation-based coverage and probabilistic target-detection missions using ground and aerial robots and multiple planners.
The case studies demonstrate task-dependent trade-offs in robot design, fleet composition, and planner selection, characterize computational behavior, and construct phase diagrams showing how optimal architectures change with coverage requirements and resource budgets.}

%% file: Figures/overview_diagram.tex
\begin{figure}
\setlength{\abovecaptionskip}{0pt}
\centering
\definecolor{teal50}{HTML}{E8FAF4}
\definecolor{teal600}{HTML}{3A9E82}
\definecolor{teal800}{HTML}{2A7A63}
\definecolor{blue50}{HTML}{EAF2FC}
\definecolor{blue400}{HTML}{6AACEC}
\definecolor{blue600}{HTML}{4A8FD4}
\definecolor{blue800}{HTML}{3B6FAA}
\definecolor{coral50}{HTML}{FDF0ED}
\definecolor{coral400}{HTML}{E8876A}
\definecolor{coral600}{HTML}{D46B50}
\definecolor{coral800}{HTML}{B05540}
\definecolor{purple50}{HTML}{F1EEFB}
\definecolor{purple400}{HTML}{A59AED}
\definecolor{purple600}{HTML}{7E72D6}
\definecolor{purple800}{HTML}{5E52AB}
\definecolor{gray50}{HTML}{F7F6F3}
\definecolor{gray200}{HTML}{D4D2CC}
\definecolor{gray400}{HTML}{A3A19B}
\definecolor{gray600}{HTML}{74726D}
\definecolor{green600}{HTML}{5EA036}
\definecolor{amber400}{HTML}{F0B95A}
\definecolor{mapbg}{HTML}{F4F3F0}
\definecolor{pathblue}{HTML}{6AACEC}
\definecolor{pathred}{HTML}{E8876A}
\definecolor{pathgreen}{HTML}{6ABF7B}
\definecolor{shapeblue}{HTML}{4A8FD4}
\definecolor{shapered}{HTML}{D46B50}
\definecolor{shapegreen}{HTML}{4AA85C}
\definecolor{framegray}{HTML}{B0AEA8}
\definecolor{textbase}{HTML}{4A4945}
\tikzset{
  every node/.style={font=\sffamily\tiny},
  outerbox/.style={
    rounded corners=3pt, draw=framegray, line width=0.5pt,
  },
  zone/.style={
    rounded corners=3pt, draw=framegray, line width=0.45pt, fill=gray600!2.5,
  },
  seclabel/.style={
    font=\sffamily\tiny\bfseries, text=textbase,
    anchor=west,
  },
  sectext/.style={
    font=\sffamily\tiny, text=gray600,
    anchor=west,
  },
  module/.style 2 args={
    rounded corners=2pt, inner ysep=1.25pt, minimum height=8pt, minimum width=30pt,
    font=\sffamily\tiny, draw=#1, fill=#2, line width=0.4pt, align=center,
    anchor=west,
  },
  alg/.style={
    rounded corners=2pt, inner ysep=1.25pt, minimum height=8pt, minimum width=22pt,
    font=\sffamily\tiny, draw=purple600!90, fill=purple50!85, text=purple800, line width=0.4pt,
    anchor=west,
  },
  arrow/.style={
    -{Stealth[length=5pt,width=3.5pt]}, textbase, line width=0.8pt,
  },
  obj/.style={
    rounded corners=2pt, inner sep=2pt, line width=0.3pt, align=center,
    font=\sffamily\tiny, draw=gray200, fill=gray50, text=gray600,
    anchor=west,
  },
  codesign/.style={
    rounded corners=3pt, minimum width=240pt, minimum height=10pt, align=center,
    font=\sffamily\footnotesize\bfseries, fill=gray600!10, draw=framegray, text=gray600,
    line width=0.5pt, anchor=center,
  },
  shape1/.style={
    regular polygon, regular polygon sides=3,
    inner sep=0, minimum size=5, fill=shapeblue,
    anchor=center,
  },
  shape2/.style={
    regular polygon, regular polygon sides=3,
    inner sep=0, minimum size=5, fill=shapered,
    anchor=center,
  },
  shape3/.style={
    regular polygon, regular polygon sides=3,
    inner sep=0, minimum size=5, fill=shapegreen,
    anchor=center,
  },
  path/.style={
    line width=0.6pt, opacity=0.55, -{Stealth[length=3pt,width=2pt]},
  },
  pareto/.style={line width=0.75pt, coral800!75!red, opacity=0.6},
}
\begin{adjustbox}{max width=1\columnwidth, center}
\begin{tikzpicture}[x=1pt, y=-1pt]

\draw[outerbox] (0, 14) rectangle (250, 200);
\node[font=\sffamily\footnotesize, text=textbase, anchor=north] at (125, 0) {Heterogeneous Multi-Robot System Design Space};

\begin{scope}[shift={(5, 18)}]
\draw[zone] (0, 12) rectangle (105, 90);
\node[font=\sffamily\scriptsize, text=textbase, anchor=north] at (55, 0) {Design Decisions};

\begin{scope}[shift={(0, 20)}]
    \node[seclabel] at (0, 0) {Robot Components};
    \node[module={coral600}{coral50}, text=coral800] at (4, 11) {Type};
    \node[module={teal600}{teal50}, text=teal800] at (36, 11) {Actuator};
    \node[module={teal600}{teal50}, text=teal800] at (68, 11) {Battery};
    \node[module={blue600}{blue50}, text=blue800] at (4, 21) {Sensor};
    \node[module={blue600}{blue50}, text=blue800] at (36, 21) {...};
\end{scope}

\begin{scope}[shift={(0, 55)}]
    \node[seclabel] at (0, 0) {Fleet Composition};
    \node[sectext] at  (8, 8)  {UAV M.};
    \node[shape1]  at   (7,8) {};
    \node[sectext] at (29, 9)  {$\times N_1$};
    \node[sectext] at  (8, 17) {UAV L.};
    \node[shape2]  at   (7,17) {};
    \node[sectext] at (29, 18) {$\times N_2$};
    \node[sectext] at  (8, 26) {Ground};
    \node[shape3]  at  (7, 26) {};
    \node[sectext] at (29, 27) {$\times N_3$};
\end{scope}

\draw[line width=0.4pt, opacity=0.2, framegray] (52.5, 60) -- (52.5, 85);

\begin{scope}[shift={(52.5, 55)}]
    \node[seclabel] at (0, 0) {Planning Algorithm};
    \node[alg] at (14, 8) {AGD};
    \node[alg] at (14, 18) {DARP};
    \node[alg] at (14, 28) {MRTA};
\end{scope}

\end{scope}

\begin{scope}[shift={(140, 18)}]
\draw[zone] (0, 12) rectangle (105, 90);
\node[font=\sffamily\scriptsize, text=textbase, anchor=north] at (55, 0) {Deployment and Evaluation};

\begin{scope}[shift={(0, 15)}]
    \draw[fill=mapbg, draw=framegray, line width=0.3pt, rounded corners=1pt] (2.5,0) rectangle (102.5, 40);
    \foreach \x in {10,20,...,90} {
    \draw[gray200, line width=0.2pt] (\x+2.5,0) -- (\x+2.5,40);
    }
    \foreach \y in {10,20,...,30} {
    \draw[gray200, line width=0.2pt] (2.5,\y) -- (102.5,\y);
    }
    \draw[pathgreen, path]
      (7.5, 35) -- (100, 35);
    \node[shape3] at (7.5, 35) {};
    \draw[pathred, path]
      (57.5, 5) -- (7.5, 5) -- (7.5, 15) -- (57.5, 15) -- (57.5, 25) -- (7.5, 25);
    \node[shape2] at (57.5, 5) {};
    \draw[pathblue, path]
      (67.5,5) -- (97.5,5) -- (97.5,15) -- (67.5,15) -- (67.5,25) -- (97.5,25);
    \node[shape1] at (67.5,5) {};
    \node[sectext] at (2.5, 5.5) {Map $M$};
\end{scope}

\begin{scope}[shift={(0, 63.5)}]
    \node[seclabel] at (0, 0) {Task Profile};
    \node[sectext] at (0, 9)  {Coverage};
    \node[sectext] at (26, 9)  {$\geq 95\%$};
    \node[sectext] at (0, 17) {Detection };
    \node[sectext] at (26, 17) {$\geq 90\%$};
\end{scope}

\draw[line width=0.4pt, opacity=0.2, framegray] (52.5, 65) -- (52.5, 85);

\begin{scope}[shift={(52.5, 63.5)}]
    \node[seclabel] at (0, 0) {Objectives};
    \node[obj] at (5, 8) {time};
    \node[obj] at (22.5, 8) {cost};
    \node[obj] at (5, 18) {energy};
\end{scope}

\end{scope}

\draw[arrow]  (115, 55) -- (135, 55);
\node[seclabel, anchor=center] at (125, 50) {deploy};

\draw[arrow] (135, 85) -- (115, 85);
\node[seclabel, anchor=center] at (125, 80) {constrain};

\begin{scope}[shift={(0, 128)}]
    \node[codesign] at (125, 0) {Task-Driven Co-Design Optimization};
\end{scope}
\draw[arrow] (57.5, 109) -- (57.5, 120);
\draw[arrow] (192.5, 109) -- (192.5, 120);
\draw[arrow] (125, 136) -- (125, 147);


\newcommand{\paretoplot}[5]{%
  \begin{scope}[shift={(#1, 150)}]
    \draw[gray400, line width=0.3pt, line join=round] (0, 5) -- (0, 40) -- (70, 40);
    \node[font=\sffamily\tiny, text=coral800, anchor=center] at (35, 45) {#2};
    \node[font=\sffamily\tiny, text=coral800, anchor=center] at (0, 0) {#3};
    \draw[pareto] plot[smooth, tension=0.55] coordinates {#4};
    \draw[pareto] plot[smooth, tension=0.55] coordinates {#5};
    \draw[only marks, mark=*, mark size=0.75pt, coral800!75!red!75] plot coordinates {#4};
    \draw[only marks, mark=*, mark size=0.75pt, coral800!75!red!75] plot coordinates {#5};
  \end{scope}
}

\paretoplot{40}{Cost}{Time}{(5,13) (10,23) (20,30) (50,36) (65,38)}{(5,20) (30,23) (45,27) (65,35)}

\draw[only marks, mark=*, mark size=0.75pt, coral800!75!red!75] plot coordinates {(64, 159)};
\node[font=\sffamily\fontsize{4}{5}\selectfont, align=center, text=gray600, anchor=west] at (63, 159) {Pareto-optimal designs};


\newcommand{\phaseplot}[3]{%
  \begin{scope}[shift={(#1,150)}]

    \begin{scope}
    
    \filldraw[
      fill=pathblue!50,
      draw=gray400,
      line width=0.3pt,
      line join=round
    ]
    (0,24) --
    (34,24) --
    (46,40) --
    (0,40) -- cycle;
    \node[font=\sffamily\fontsize{3}{4}\selectfont, align=center, text=pathblue!80!black, anchor= center] at (9, 37.25) {Algorithm 2};
    
    \filldraw[
      fill=pathgreen!50,
      draw=gray400,
      line width=0.3pt,
      line join=round
    ]
    (18,5) --
    (70,5) --
    (70,40) --
    (45,40) -- cycle;
    \node[font=\sffamily\fontsize{3}{4}\selectfont, align=center, text=pathgreen!80!black, anchor= center] at (60.5, 37.25) {Algorithm 3};
    
    \filldraw[
      fill=pathred!50,
      draw=gray400,
      line width=0.3pt,
      line join=round
    ]
    (0,5) --
    (28,5) --
    (33.5,25) --
    (0,25) -- cycle;
    \node[font=\sffamily\fontsize{3}{4}\selectfont, align=center, text=pathred!80!black, anchor= center] at (9, 22.25) {Algorithm 1};
        
    \end{scope}

    \draw[gray400, line width=0.3pt] (0,5) -- (0,40);
    \draw[gray400, line width=0.3pt] (0,40) -- (70,40);

    \node[
      font=\sffamily\tiny,
      text=coral800,
      anchor=center
    ] at (35,45) {#2};

    \node[
      font=\sffamily\tiny,
      text=coral800,
      anchor=center
    ] at (70,0) {#3};

  \end{scope}
}

\phaseplot{140}{Time}{Energy}

\node[font=\sffamily\fontsize{4}{5}\selectfont, align=center, text=gray600, anchor=west] at (170, 159) {Dominance regions};

\end{tikzpicture}
\end{adjustbox}
\caption{
\rev{
Overview of the proposed task-driven co-design approach for heterogeneous multi-robot systems. 
Design decisions spanning robot components, fleet composition, and planning algorithm selection are jointly optimized through the monotone co-design framework subject to deployment constraints including map geometry, task profiles, and performance objectives.
}
}
\label{fig:overview}
\end{figure}
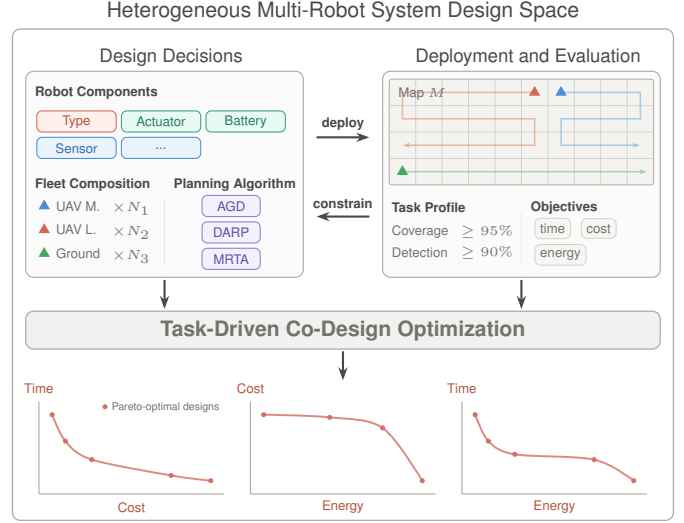

%% file: Sections/2_related_work.tex
\section{Related Work}
\label{sec:related_work}

\subsection{Multi-Agent Robotic Systems and Coverage/Search Planning}
\rev{Multi-robot systems have been studied for environmental monitoring~\cite{merino2012unmanned,mesbahi2019team,rudolph2021range}, logistics~\cite{wurman2008coordinating}, inspection~\cite{breitenmoser2010magnebike}, disaster response~\cite{debusk2010unmanned}, and search-and-rescue~\cite{liu2015supervisory,queralta2020collaborative,kazemdehbashi2024adaptive}.
Coverage and search planners typically assume a given fleet and sensing and motion capabilities~\cite{galceran2013survey,almadhoun2019survey}.
Representative methods include heterogeneous task allocation with ergodic planning~\cite{hu2022large}, \gls{acr:darp} for homogeneous teams~\cite{kapoutsis2017darp}, and \gls{acr:agd} for heterogeneous UAV coverage~\cite{kazemdehbashi2024adaptive}.
Other cooperative-search studies examine the benefits of heterogeneity~\cite{maza2007multiple,kantaros2015distributed,surmann2019integration,queralta2020collaborative}, but generally treat fleet capabilities as inputs rather than joint hardware--software design variables.}

\subsection{Hardware-Aware Planning and Robot-Level Co-Design}
\rev{Some planning algorithms consider coverage under sensing limitations~\cite{lewis2018guaranteed} and minimal-robot design~\cite{shell2021design}.
The authors of~\cite{carlone2019robot} formulate sensing, actuation, and computing-module selection as resource-constrained combinatorial optimization.
Fleet level designs combine graph-grammar-based AUV design, task-specific vehicle optimization, and heterogeneous fleet selection for simulated surveying \cite{zhao2022graph}.
These approaches provide valuable ingredients, but they typically focus on either (i) optimizing the design of an individual platform for a fixed task model, or (ii) optimizing planning/task allocation for a fixed platform and fleet.}

\subsection{Monotone Computational Co-Design}
Our work builds on the monotone theory of co-design~\cite{censi2015mathematical, zardini2023dissertation, censi2024}, which models complex systems as interconnected design problems with partially ordered \F{functionalities} and \R{resources}.
\rev{In co-design, modular subsystems can be composed while preserving formal guarantees, with partial orders enabling multi-objective Pareto analysis and efficient exploration of design choices.}
Monotone co-design has been applied to a range of domains, including co-design of autonomous systems that jointly reasons about hardware and autonomy stacks~\cite{milojevic2025codei}, strategic co-design in Formula~1~\cite{neumann2024codesignf1}, and mobility-system design that couples vehicle characteristics with fleet sizing and infrastructure decisions~\cite{zardini2022codesignmobility}.
\rev{At the multi-agent level, Wilhelm and Napp~\cite{wilhelm2025monotone} introduce monotone subsystem decomposition, a constraint-programming approach that reuses subsystem Pareto fronts under consistency conditions while preserving global Pareto optimality, and apply it to component-level quadcopter design and fleet-level package-delivery scheduling.
Complementary work also addresses scalable co-design at the swarm level through dynamic speciation~\cite{wilhelm2026swarmcode}.}

\subsection{Gap and Positioning}

\rev{Existing work has demonstrated the value of coupling robot-level design with fleet composition and task allocation in specific application domains.
However, systematic variation across robot designs, fleet structures, planning and execution methods, task requirements, and environment classes within a common queryable model remains comparatively underdeveloped.
Our contribution is then a \emph{general, formal, and compositional framework} that integrates heterogeneous robot design, fleet composition, planning, execution, and task evaluation.
This integration provides the computational substrate needed to move from isolated design studies to robot phase-diagrams:
for different tasks and environments, what system architectures become optimal.}

%% file: Sections/3_mathematical_preliminaries.tex
\section{Mathematical preliminaries}
\label{sec:math-preliminary}

\subsection{Sets and Functions}
\label{subsec:sets-functions}
\rev{We write $\mapf:\SetA\to\SetB,\, \setelmA \mapsto \mapf(\setelmA)$ for a map and its element-wise evaluations, $\mapf\inv(\subsetYB)=\setWithArg{\setelmA\in\SetA}{\mapf(\setelmA)\in\subsetYB}$ for its pre-image, and, for $\mapg\colon\SetB\to\SetC$,
$\mapg\mapFollowing\mapf\colon\SetA\to\SetC$ for their composition.
The Cartesian product $\SetA\setproduct\SetB$ contains pairs $\tup{\setelmA,\setelmB}$.
For maps $\mapf:\SetA\to\SetB$ and $\mapg:\SetA'\to\SetB'$, $(\mapf\setproduct\mapg)(\setelmA,\setelmA')=\tup{\mapf(\setelmA),\mapg(\setelmA')}$.
For $\maph:\SetA\setproduct\SetB\to\SetC$, $\maph(-,\setelmB)$, partial evaluation at $\setelmB$ gives the map $\maph(-,\setelmB) \colon \SetA  \to \SetC,\, \setelmA \mapsto \maph(\setelmA, \setelmB)$.}

\subsection{Background on Orders}

\rev{
A \emph{\gls{abk:poset}} is a tuple $\mathcal{P} =\tup{ P,\preceq_\mathcal{P}}$, where $P$ is a set and~$\preceq_\mathcal{P}$ is a partial order (a reflexive, transitive, and antisymmetric relation). 
If clear from context, we use~$\posetP$ for a \gls{abk:poset}, and~$\posetleq$ for its order.
}

\rev{
The \emph{opposite} of a \gls{abk:poset}~$\mathcal{P} = \tup{ P,\preceq_\mathcal{P}}$ is the poset $\mathcal{P}\op \defeq\tup{ P,\posetleq_{\mathcal{P}}\op }$ with the same elements and reversed ordering:
$
\poselxP \posetleq_{\mathcal{P}}\op \poselyP \Leftrightarrow 
\poselyP \posetleq_{\mathcal{P}}\poselxP
$.
Given \glspl{abk:poset} $\tup{P,\preceq_{\mathcal{P}}}$ and $\tup{Q,\preceq_{\mathcal{Q}}}$, their \emph{product} $\tup{P\times Q,\preceq_{\mathcal{P}\times \mathcal{Q}}}$ is the poset with
\begin{equation*}
    \tup{\poselxP,\poselxQ}\preceq_{\mathcal{P}\times \mathcal{Q}}\tup{\poselyP,\poselyQ} \Leftrightarrow (\poselxP \preceq_{\mathcal{P}} \poselyP) \wedge (\poselxQ \preceq_\mathcal{Q} \poselyQ).
\end{equation*}
}

\rev{
A subset $A \subseteq \posetP$ of a \gls{abk:poset} is called an \emph{antichain} if every pair of elements in $A$ is non-comparable.
A map $f\colon \posetP \to \posetQ$ between \glspl{abk:poset} is \emph{monotone} if $x\preceq_\mathcal{P} y$ implies $f(x) \preceq_\mathcal{Q} f(y)$. Monotonicity is preserved by composition and products.
}

%% file: Sections/4_problem_formulation.tex
\section{Problem Formulation}
\label{sec:problem_formulation}

This section introduces the terminology and abstractions used to model and design heterogeneous multi-robot systems.
We consider a modular pipeline in which (i) a \emph{planner} assigns waypoints, (ii) an \emph{executor} produces dynamically feasible trajectories consistent with robot capabilities and operating-time budgets, and (iii) an \emph{evaluator} computes task-relevant performance metrics from the resulting trajectories.
A \emph{task profile} specifies which performance metric levels are required.

\paragraph*{Pipeline abstraction}
Given a fleet $\fleet$ operating in a map $\map$, the overall pipeline can be summarized as
\begin{equation*}
\tup{\fleet,\map}
\xrightarrow{\planner}
\collwpoints
\xrightarrow{\exec(\cdot,\tbudvec)}
\colltraj
\xrightarrow{\eval}
\mathbf{m}
\xrightarrow{\task}
\{0,1\},
\end{equation*}
where $\collwpoints$ is a fleet-level waypoint assignment, $\tbudvec$ is \rev{per-robot operating-time budgets}, $\colltraj$ is the resulting collection of trajectories, $\mathbf{m}$ is the vector of performance metric values, and $\task$ \rev{evaluates} whether the required task profile is satisfied.

\subsection{Problem Overview}
\label{subsec:problem_formulation_overview}
\rev{We begin by introducing necessary abstractions in heterogeneous robot system design.}

\begin{definition}[Map]
\label{def:map}
A \emph{map}~$\map$ is the space of interest in which the fleet operates. 
We model it as a subset of the Euclidean space~$\map \subset \mathbb{R}^k$, where typically~$k\in\{2,3\}$ depending on whether the workspace is planar or volumetric.
\end{definition}

\begin{definition}[Robot]
\label{def:robot}
A \emph{robot}~$\robot$ is a tuple~$\robot=\tup{\dprop, \sensing}$, 
where~$\dprop$ denotes the robot's dynamical \rev{capabilities (i.e., motion-related properties)} and~$\sensing$ denotes its sensing capabilities.
Moreover, given a resource \gls{abk:poset}~$\tup{\resrobot, \preceq_\resrobot}$, there exists a map~$\res \colon \mathcal{R} \to \resrobot$ from the set of robot designs to required resources, where~$\res(\robot)$ quantifies the resources needed to realize~$\robot$ (e.g., cost, power, energy capacity).
\end{definition}
The exact definition of the underlying objects that constitute a robot will be given in \Cref{subsec:abstraction_details}.
Intuitively, a robot exposes \emph{capabilities} $\tup{\dprop,\sensing}$ to planning/execution, while requiring corresponding resources $\res(\robot)$ to be realized.

\begin{definition}[Fleet]
\label{def:fleet}
A \emph{fleet} $\fleet$ of size~$n$ is \rev{a tuple} of robots~$\fleet= \tup{\robot_1, \ldots, \robot_n}$.
\end{definition}
\rev{\Cref{def:fleet} allows flexible heterogeneity: robots may differ in dynamics, sensing, and underlying state/control representations, while the tuple with indices (rather than a set) allows fleets to include multiple identical robots.}

\subsubsection{Planner}
\rev{Given environment and fleet,} a planner assigns waypoint sequences to robots so that, collectively, the fleet can attempt to accomplish tasks on map $\map$.

\begin{definition}[Planner]
\label{def:planner}
A \emph{planner} is a map~$\planner:\fleet \times \map \to \collwpoints$, where~$\collwpoints$ is a fleet-level collection of waypoint sequences (formalized in \cref{def:collection_waypoints}).
\end{definition}

We use the term \emph{task} to denote the abstract mission objective (e.g., area coverage, search-and-rescue, leakage detection).
\rev{
Planners often tackle specific tasks, which we view as \emph{internal to the planners}.
Namely, the planners don't necessarily tackle the actual tasks captured by downstream metrics and task profiles (\Cref{def:metric,def:task_profile}).
This lets the framework support a broad range of tasks independent of the planners.
}

\begin{assumption}[Centralized offline planning]
\label{ass:centralized_planning}
Planning is performed centrally and offline. 
The framework does not explicitly model inter-robot communication or distributed information exchange. 
Accordingly, the planner has global access to the fleet’s effective dynamical and sensing capabilities and the map when computing waypoint assignments.
\rev{The resulting waypoint assignments are assumed to be available to the robots before mission execution.
During execution, each robot tracks its assigned waypoints locally. Online communication, coordination, and replanning are not modeled.}
\end{assumption}

\subsubsection{Executor}
The \emph{executor} generates \rev{dynamically and time-wise feasible trajectories from waypoints}.

\begin{definition}[Trajectory]
\label{def:trajectory}
A \emph{trajectory}~$\traj$ is a map~$\traj \colon [0,T_i] \to \map$, which associates to each time~$t\in[0,T_i]$ a position~$\traj(t)\in \map$ in the environment.
\end{definition}

\begin{definition}[Collection of trajectories]
\label{def:collection_trajectory}
A \emph{collection of trajectories} for a fleet~$\fleet=\tup{\robot_1,\ldots,\robot_n}$ is the tuple~$\colltraj = (\traj_1,\ldots,\traj_n),$ where $\traj_i$ is the trajectory executed by robot $\robot_i$.
\end{definition}

For each robot~$\robot_i$, let~$\tbud_i$ denote the maximum allowed operating duration.
We aggregate these budgets for all robots into~$\tbudvec=\tup{\tbud_1,\ldots,\tbud_n}\in \mathbb{R}_{\ge 0}^n$.

\begin{definition}[Executor]
\label{def:executor}
An \emph{executor} is a map
\begin{equation*}
\exec:\fleet \times \collwpoints \times \mathbb{R}_{\ge 0}^n \to \colltraj
\end{equation*}
that takes a fleet~$\fleet$, waypoint collection~$\collwpoints$, and time budget vector~$\tbudvec$, and returns a collection of dynamically feasible trajectories $\colltraj=\tup{\traj_1,\ldots,\traj_n}$ such that each $\traj_i$ is defined on $[0,T_i]$ with $T_i \le \tbud_i$ for all $i=1,\ldots,n$.
\end{definition}

\rev{Increasing $\tbud_i$ preserves every previously feasible trajectory because its duration still satisfies $T_i\leq\tbud_i$.}

\subsubsection{Evaluator, Metrics, and Task Profiles}
The evaluator computes performance metrics from the executed trajectories. 
\rev{Performance metrics may incorporate additional (possibly hidden from the planners) environment information (e.g. target locations) as internal parameters.}

\begin{definition}[Performance metric]
\label{def:metric}
A \emph{performance metric} is a monotone map
\begin{equation*}
m:\fleet \times \colltraj \times \map \to \perfspace,
\end{equation*}
where~$\tup{\perfspace,\preceq_{\perfspace}}$ is a \gls{abk:poset} representing the performance space.
Without loss of generality, metrics are defined so that ``better'' performance corresponds to larger elements under $\preceq_{\perfspace}$.
\end{definition}

\begin{remark}
Typical examples \rev{of metrics} include coverage percentage \rev{(with $\perfspace=[0, 100]$ ordered by $\leq$)} and number of detected targets \rev{(with $\perfspace=\mathbb{N}_0$ ordered by $\leq$)}.
\end{remark}

The evaluator can represent the effects of time-varying sensing phenomena, including occlusions and clutter, within the performance assessment.
Given $q$ performance metrics $m_1,\ldots,m_q$ (with codomains $\perfspace_1,\ldots,\perfspace_q$), the evaluator returns their values as a vector (performance vector).

\begin{definition}[Evaluator]
\label{def:evaluator}

The \emph{evaluator} is the map
\begin{equation*}
\eval(\fleet,\colltraj,\map) \coloneq
\bigl(m_1(\fleet,\colltraj,\map),\ldots,m_q(\fleet,\colltraj,\map)\bigr).
\end{equation*}
\end{definition}

A task profile specifies required performance levels on the performance metric vector returned by the evaluator.

\begin{definition}[Task profile]
\label{def:task_profile}
\rev{
Let $\mathbf{m} = \tup{m_1, \ldots, m_q}\in \perfspace_1\times\cdots\times\perfspace_q$ denote the performance metric-value vector produced by the evaluator, i.e.\ $\mathbf{m}=\eval(\fleet,\colltraj,\map)$.
A \emph{task profile} is a monotone map $\task:\perfspace_1\times\cdots\times\perfspace_q\to\{0,1\}$.
We say the task profile is \emph{satisfied or fulfilled} if~$\task(\mathbf{m})=1$.
}
\end{definition}

\begin{remark}
\rev{
Montonicity in \cref{def:task_profile} makes sure feasibility is retained with better performances.
Thresholds define common task profiles: let $\boldsymbol{\lambda}=\tup{\lambda_1,\ldots,\lambda_q}$ be the thresholds with~$\lambda_i\in\perfspace_i$,
then a task profile to satisfy them is given by
\begin{equation*}
    \task(\mathbf{m}) =
    \begin{cases}
        1, & \text{if } m_i \succeq_{\perfspace_i} \lambda_i \quad \forall i = 1,\ldots,q, \\
        0, & \text{otherwise}.
    \end{cases}
\end{equation*}
For example, consider a coverage task with two performance metrics:
$m_1 \in \tup{[0, 100],\leq}$ is the covered percentage of $\map$ and~$m_2 \in \tup{\mathbb{N}_0,\leq}$ is the number of detected targets.
A threshold $\boldsymbol{\lambda}=\tup{95,2}$ requires at least $95\%$ coverage and at least two detections.
}
\end{remark}

\subsection{Abstraction Details}
\label{subsec:abstraction_details}

\rev{After the overview of the problem and component abstractions, we introduce the model details.}

\subsubsection{Robot Modeling}
We now expand the definitions of $\dprop$, $\sensing$, and the resource space.

\begin{definition}[Robot's dynamical properties]
\label{def:robot_dyn_prop}
A robot's \emph{dynamical properties} are described by a tuple
\begin{equation*}
    \dprop = \tup{d_1, d_2, \ldots, d_w},
\end{equation*}
where each component $d_i\in D_i$ is a dynamical parameter, and each~$\tup{D_i,\preceq_{D_i}}$ is a \gls{abk:poset}.
The partial orders are chosen so that~$d_i \preceq_{D_i} d_i'$ means ``$d_i'$ is at least as capable as $d_i$'' along that dimension.
\end{definition}

Examples include bounds such as maximum speed $v_{\max}$, maximum acceleration $a_{\max}$, or minimum turning radius $r_{\text{turn}}$.
These parameters may enter a robot model as constraints on admissible motion.
\rev{As a concrete example, motion of robots maybe governed by a parameterized model}
\begin{align}
\label{eq:robot_model}
    \dot{x}(t) = F(x(t), u(t); \dprop),
\end{align}
where $x$ is the state, $u$ is the control input, and $\dprop$ parametrizes the admissible state/input evolution (e.g., via $|v(t)|\le v_{\max}$, $|a(t)|\le a_{\max}$, or curvature bounds induced by $r_{\text{turn}}$).
Thus, $\dprop$ determines the set of dynamically feasible motions under \Cref{eq:robot_model}.
\rev{To represent different dynamical models,  $\dprop$ may include a categorical model class $d_{\mathrm{type}}$ and type-specific parameters; irrelevant coordinates are then ignored by $F$.}

\begin{definition}[Robot's sensing capabilities]
\label{def:robot_sens_prop}
A robot's \emph{sensing capabilities} are described by a tuple~$\sensing=\tup{s_1,\ldots,s_v},$ where each component $s_j\in S_j$ encodes a sensing-related parameter, and each $\tup{S_j,\preceq_{S_j}}$ is a \gls{abk:poset}.
As for dynamics, the partial orders are chosen so that $s_j \preceq_{S_j} s_j'$ means ``$s_j'$ is at least as capable as $s_j$'' along that sensing dimension.
\end{definition}

Examples include sensing range, resolution, field-of-view, and detection certainty.
To represent different sensing modalities and sensor suites, one may include a categorical or set-valued component inside $\sensing$ (e.g., modeling the sensor suite as a subset of available modalities ordered by set inclusion, so that \texttt{camera} $\preceq$ \texttt{camera+LiDAR}).

Finally, the resource map $\res(\cdot)$ associates to each robot design the resources required to realize its capabilities.

\begin{definition}[Robot resource space]
\label{def:robot_resources}
Let $\tup{\resposet_i,\preceq_{\resposet_i}}$, for $i = 1,\dots,k$, \glspl{abk:poset} describing individual resource dimensions (e.g., cost, power, energy capacity). 
The \emph{robot resource space} is the product \gls{abk:poset} $\resrobot \coloneq \prod_{i=1}^{k} \resposet_i$, equipped with the standard product order.
\end{definition}

\subsubsection{Planner Modeling}
\rev{The planner generates, for each robot $\robot_i$ in the fleet $\fleet$, a sequence of waypoints such that the fleet accomplishes the given task on the map $\map$.
We first define the notion of a robot’s waypoints.}

\begin{definition}[Waypoints of a robot]
\label{def:waypoints_robot}
The \emph{waypoint sequence} assigned to robot $\robot_i$ is \rev{a tuple}
\begin{equation*}
\wpoints_i = \tup{w_{i,1},\ldots,w_{i,p_i}},
\qquad w_{i,j}\in \map,
\end{equation*}
where $p_i\in\mathbb{N}$ may depend on the robot and the planning instance.
\end{definition}

\begin{definition}[Collection of waypoints]
    \label{def:collection_waypoints}
    The \emph{collection of waypoints} for a fleet $\fleet = \tup{\robot_1,\dots,\robot_n}$ is the tuple~$\collwpoints = \tup{\wpoints_1,\ldots,\wpoints_n},$ where $\wpoints_i$ is the waypoint sequence assigned to robot $\robot_i$.
\end{definition}

\subsection{Planning for multiple maps}
In many applications, \rev{the multi-agent system must generalize across multiple environments rather than a single fixed map.
We capture this by introducing a finite set of maps.}

\begin{definition}[Set of maps]
    \label{def:setofmaps}
A \emph{set of maps} is a finite set~$\setofmaps = \{\map_1,\ldots,\map_p\}$, where each $\map_i$ is a map in the sense of \Cref{def:map}.
\end{definition}

Applying a planner to each $\map_i$ yields a corresponding set of waypoint collections.

\begin{definition}[Set of waypoint collections]
\label{def:set_waypoint_collections}
Given a fleet $\fleet$ and planner $\planner$, the \emph{set of waypoint collections} induced by  $\setofmaps$ is
\begin{equation*}
\setcollwpoints = \{\collwpoints_1,\ldots,\collwpoints_p\},
\qquad
\collwpoints_i \coloneq \planner(\fleet,\map_i).
\end{equation*}
\end{definition}

Similarly, applying the executor to each waypoint collection $\collwpoints_i$ yields a corresponding set of trajectory collections.

\begin{definition}
    \label{def:set_trajectory_collections}
    Given a fleet $\fleet$ and a time budget vector $\tbudvec$ for the maximum allowed operating duration for each robot, the \emph{set of trajectory collections} induced by the \emph{set of waypoint collections} is
    \begin{equation*}
        \setcolltraj = \{ \colltraj_1, \ldots, \colltraj_p \},
        \qquad
        \colltraj_i \coloneq \exec(\fleet, \collwpoints_i, \tbudvec).
    \end{equation*}
\end{definition}

\begin{definition}[Mission completion time]
\label{def:mission_completion_time}
\rev{
Fix a task profile $\task(\cdot)$ on selected performance metrics $(m_j)_{j\in I}$.
For $\colltraj = \tup{\traj_1, \dots, \traj_n}$ with $\mathrm{dom}(\traj_i)=[0,T_i]$, define 
$\colltraj_{\leq t} \coloneq (\traj_i |_{[0,\, \min\{t,T_i\}]})_{i=1}^{n}$.
Let $\mathbf{m}(t) \coloneq (m_j(\fleet, \colltraj_{\leq t}, \map))_{j\in I}$ denote the corresponding metric vector at time $t$.
The mission completion time is
}
\begin{equation*}
\rev{
\missiontime^{\task}(\fleet, \colltraj, \map)
\coloneq
\inf\left\{
t \geq 0
\;\mid\;
\task\big( \mathbf{m}(t) \big) = 1
\right\}.
}
\end{equation*}
\rev{
We take $\inf\emptyset=+\infty$, indicating an unfulfilled task, and equip the performance space $[0,+\infty]$ with the order $\ge$.
}
\end{definition}

\begin{definition}[Evaluation over multiple maps]
\rev{
For each map $\map_k$ and corresponding trajectory collection $\colltraj_k$, let $I_k$ index the selected performance metrics and $\task_k$ denote a task profile on those metrics.
The multi-map evaluator collects the per-map performance vectors
}
\begin{align*}
\rev{\eval_{\setofmaps}(\fleet,\setcolltraj,\setofmaps)}
&\rev{\coloneq
\bigl(\mathbf{m}_k\bigr)_{k=1}^{p},} \\
\qquad
\rev{\mathbf{m}_k}
&\rev{\coloneq
\bigl(m_j(\fleet,\colltraj_k,\map_k)\bigr)_{j\in I_k}.}
\end{align*}
\rev{
Both the selected metrics and task profiles may differ across maps.
Requiring fulfillment on all maps imposes $\task_k(\mathbf{m}_k)=1$ for every $k$.
Mission completion time is evaluated separately on each map as $\missiontime^{\task_k}(\fleet,\colltraj_k,\map_k)$.
}
\end{definition}

\begin{remark}
This construction enables evaluating a fixed fleet (and fixed planner) across multiple environments.
More generally, one may compare different planners by evaluating their induced waypoint-collection sets over the same $\setofmaps$, while keeping the fleet design fixed.
\end{remark}

%% file: Sections/5_formal_codesign_model.tex
\section{Formal Co-Design Model}
\label{sec:codesign-formalization}
\rev{Building on \Cref{sec:problem_formulation}, we cast heterogeneous multi-robot system design as an optimization problem that jointly determines (i) the design of the individual robots, (ii) the fleet composition, and (iii) the planning procedure, satisfying a desired task profile while trading off global resources such as cost and energy.}
To do so in a principled and compositional way, we leverage the monotone theory of co-design~\cite{zardini2023dissertation}.

\rev{\Cref{subsec:background_codesign} summarizes the monotone theory of co-design~\cite{zardini2023dissertation}, \Cref{subsec:general_fleet_design,subsec:general_planner_design,subsec:general_executor_design,subsec:general_evaluator_design} model the fleet composer, planner, executor, and evaluator as \glspl{acr:mdpi}, and \Cref{subsec:global_codesign_problem} interconnects them into the global co-design problem of \Cref{fig:general_codesign_model}.}

\begin{figure*}[th!]
\centering
\setlength{\abovecaptionskip}{0pt}
\includegraphics[width=.95\linewidth]{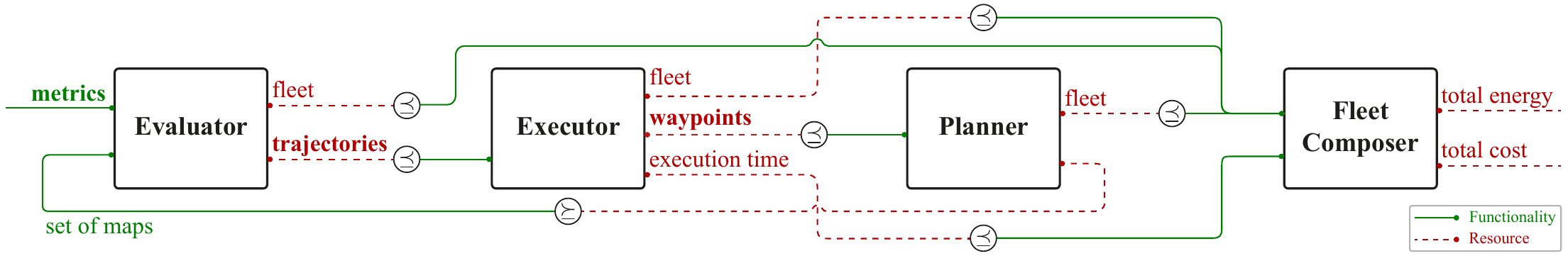}
\caption{
Overview of our general co-design model. Bold wires indicate higher-order collections.
\rev{
\F{Solid Green wires} represent \F{functionality} interfaces, while \R{red dashed wires} represent \R{resource} interfaces.
}
}
\label{fig:general_codesign_model}
\end{figure*}

\subsection{Background on the Monotone Theory of Co-Design}
\label{subsec:background_codesign}

\paragraph{Formulating Co-Design Problems}
The core modeling primitive of monotone co-design is the \emph{monotone design problem with implementation} (\gls{acr:mdpi}).

\begin{definition}[MDPI]
\label{def:mdpi}
Let \glspl{abk:poset}~$\setOfFunctionalities{}$ and $\setOfResources{}$ denote, respectively, the spaces of \F{functionalities} and \R{resources}.
A \gls{acr:mdpi} is a triple~$d \coloneq \tup{\setOfImplementations{d},\prov,\req}$, where~$\setOfImplementations{d}$ is a set of implementations and~$\prov: \setOfImplementations{d}\to \setOfFunctionalities{}$,~$\req: \setOfImplementations{d}\to \setOfResources{}$ assign to each implementation the functionality it provides and the resources it requires.
We depict this as
\begin{equation*}
        \setOfFunctionalities{} \xleftarrow{\prov} \setOfImplementations{d} \xrightarrow{\req} \setOfResources{},
\end{equation*}
and compactly write~$d\colon \setOfFunctionalities{}\to\setOfResources{}$.
An implementation~$i\in \setOfImplementations{d}$ is \emph{feasible} for a desired functionality $\F{f}\in\setOfFunctionalities{}$ under an available resource budget $\R{r}\in\setOfResources{}$ if
\begin{equation*}
    \prov(i) \succeq_{\setOfFunctionalities{}} \F{f}
    \quad\text{and}\quad
    \req(i) \preceq_{\setOfResources{}} \R{r}.
\end{equation*}
Equivalently, each \gls{acr:mdpi} induces the feasibility map
\begin{equation*}
    \begin{aligned}
        \bar{d}\colon \setOfFunctionalitiesOp{} \times \setOfResources{} &\to \tup{\powerset{\setOfImplementations{d}},\subseteq},\\
        \tup{\F{f}^*,\R{r}} &\mapsto
        \bigl\{i \in \setOfImplementations{d} \mid
        (\prov(i) \succeq_{\setOfFunctionalities{}}\F{f}) \wedge (\req(i)\preceq_{\setOfResources{}}\R{r})\bigr\},
    \end{aligned}
\end{equation*}
\rev{where $(\cdot)\op$ denotes the opposite of a \gls{abk:poset}}.
\end{definition}

In co-design diagrams, we depict an \gls{acr:mdpi} as a block with green wires on the left (\F{functionalities}) and red dashed wires on the right (\R{resources}), as illustrated in \Cref{fig:general_codesign_model}.

\begin{remark}[Monotonicity]
\label{rem:mdpi_monotonicity}
The map~$\bar d$ is monotone in both arguments.
Indeed, if fewer \F{functionalities} are demanded (i.e.,~$\F{f'}\preceq \F{f}$), then every implementation that satisfies~$\F{f}$ also satisfies $\F{f'}$, and the feasible implementation set can only grow:~$\bar d(\F{f'}^*,\R r) \supseteq \bar d(\F f^*,\R r)$.
Similarly, if more \R{resources} are available (i.e.,~$\R{r'}\succeq \R{r}$), then any implementation feasible under~$\R r$ remains feasible under~$\R{r'}$, so~$\bar d(\F f^*,\R{r'}) \supseteq \bar d(\F f^*,\R r)$.
\end{remark}

\rev{
\Cref{rem:mdpi_monotonicity} applies to the \F{functionality} and \R{resource} \emph{interfaces}, \emph{not} the internal models.
Practical validness of an \gls{acr:mdpi} is \emph{not} embedded in \cref{def:mdpi}: designer has to make sure every implementation $i$ can deliver every \F{functionality} below $\prov(i)$ with every \R{resource} above $\req(i)$, while the implementation set and the maps $\prov$ and $\req$ are not assumed ordered or monotone.
For example, a voltage measurer may have poor accuracy at low voltage from noise and at high voltage from saturation, so its internal model is \emph{non-monotone} in the natural order on voltages.
Its \gls{acr:mdpi} interface can still be monotone: with accuracy as the \F{functionality} and \emph{intervals of voltages} as the \R{resource}, ordered by set inclusion, a smaller interval always promises better accuracy.
From the compositional perspective, the partial order on an interface is a consensus by all related components: components \emph{providing} elements in it can provide lower ones for free, while those \emph{requiring} elements can behave the same with higher ones.
}

\begin{remark}[Populating the models]
\label{rem:populating_models}
The co-design framework is agnostic to how one specifies the implementation set~$\setOfImplementations{d}$.
In practice, implementations can be populated via analytic relations (e.g., cost models), numerical computation (e.g., solving optimal control problems), or in a data-driven or on-demand fashion (e.g., via simulations or by solving instances of planning problems).
For concrete examples in mobility and autonomy systems, see~\cite{zardini2022codesignmobility, zardini2023dissertation, milojevic2025codei, neumann2024codesignf1}.
\end{remark}

\rev{
A core tenet of the co-design theory is that \glspl{acr:mdpi} can be \emph{composed} in series, where the functionality provided by one is a resource required by another, in parallel, where two operate independently, and in feedback, where a functionality is connected to a resource of the same component and enforces a consistency constraint.
Every wire from a provided functionality $\F{f}$ to a required resource $\R{r}$ carries the constraint $\R{r}\preceq \F{f}$, so a downstream module cannot require more than the upstream module provides.
}
\rev{When every component is valid, the interconnected multigraph is itself a valid \gls{acr:mdpi}~\cite{zardini2023dissertation}.}
\footnote{Formally, the co-design primitives endow the category of design problems with a traced monoidal structure (locally posetal), which guarantees well-posedness of feedback compositions; see~\cite{zardini2023dissertation} for details.}

\paragraph{Solving Co-Design Problems}
Given a \gls{acr:mdpi}, one can query it in two dual ways.
A \texttt{FixFunMinRes} query fixes \rev{a desired \F{functionality}} and returns the Pareto-optimal (i.e., minimal) \R{resources} required to provide them.
A \texttt{FixResMaxFun} query fixes available \R{resources} and returns the Pareto-optimal (i.e., maximal) \F{functionalities} achievable under that budget.

\begin{definition}
\label{def:h_map}
Given a \gls{acr:mdpi}~$d$, one defines monotone maps
\begin{itemize}
    \item $h_d\colon \setOfFunctionalities{}\to \mathsf{A}\setOfResources{}$, mapping a functionality to the \emph{minimum} antichain of resources providing it;
    \item $h_d'\colon \setOfResources{}\to \mathsf{A}\setOfFunctionalities{}$, mapping a resource budget to the \emph{maximum} antichain of functionalities provided by it.
\end{itemize}
\end{definition}

\rev{
One substantial result of co-design is that component solution maps compose in series and parallel to system solution maps.
For feedback, assuming complete posets and Scott-continuity, Kleene's fixed point theorem gives an algorithm that converges to the set of optimal solutions or certifies infeasibility~\cite{censi2015mathematical, zardini2023dissertation}.
Importantly, this compositional approach avoids brute-force enumeration of all design combinations: its computational cost scales with the size of intermediate solutions and the height of feedback \gls{abk:poset}, rather than combinatorially with the number of design decisions~\cite{censi2015mathematical,zardini2023dissertation}.
}

\rev{
Practical queries may involve a mixture of \F{functionality} and \R{resource} requirements.
Taking the opposite allows us to switch the role of a \gls{abk:poset} interface and compute those mixing queries.
}
\begin{definition}
\label{def:fun_flip}
\rev{
For a poset $\mathcal{Q}$, define the functionality-to-resource wire-flip design problem $\operatorname{Flip}_{\mathcal{Q}} \colon \mathbf{1} \to \mathcal{Q} \times \mathcal{Q}^{\mathrm{op}}$ by
\begin{equation*}
\operatorname{Flip}_{\mathcal{Q}}
\big(*, (q_1,q_2^*) \big) = \top
\iff 
q_2 \preceq_{\mathcal{Q}} q_1.
\end{equation*}
Composing it with $d \colon \mathcal{F}\times\mathcal{Q} \to \mathcal{R}$ moves the $\mathcal{Q}$ interface to the resource side, yielding a design problem $\mathcal{F}\to\mathcal{R}\times\mathcal{Q}^{\mathrm{op}}$.
}
\end{definition}

\subsection{The Fleet Design Problem}
\label{subsec:general_fleet_design}
\rev{We cast the fleet-design layer of \Cref{sec:problem_formulation} as a composite \gls{acr:mdpi}, the \textbf{Fleet Composer}, which jointly selects each robot's design and the number of robots in the fleet.
It provides the \F{fleet} and associated \F{execution-time budgets} while requiring the \R{total resources} needed to realize and operate the fleet.
Internally, it interconnects one \textbf{Robot} \gls{acr:mdpi} per available robot ``slot'' (see \Cref{fig:general_fleet_composer_detailed}).}

\subsubsection{Robot MDPI}
\rev{The \textbf{Robot} \gls{acr:mdpi} in \Cref{fig:general_codesign_robot_mcdp} selects hardware and low-level control components (e.g., battery, actuation, and sensing) and exposes quantities coupling these choices to planning, execution, and fleet-level trade-offs.}

\begin{assumption}[Robot reduction]
\label{ass:robot_modeling_determination}
\rev{At system level, robots are fully characterized by $\tup{\texttt{cost},\;\texttt{power},\;\texttt{capacity},\;\dprop,\;\sensing}$.
Their implementation details affect planning and execution only through $\dprop$ and $\sensing$, and fleet-level resource trade-offs only through \texttt{cost}, \texttt{power}, and \texttt{capacity}.}
\end{assumption}

\rev{Under \Cref{ass:robot_modeling_determination}, the \textbf{Robot} \gls{acr:mdpi} provides the \F{dynamical properties} (\Cref{def:robot_dyn_prop}), \F{sensing capabilities} (\Cref{def:robot_sens_prop}), and \F{available energy capacity} (Wh), while requiring \R{acquisition cost} (USD) and \R{operational power} (W),
instantiating \Cref{def:robot_resources} for this model.
Other applications may expose additional properties (e.g., mass, payload, or compute) or absorb some variables into other components.
Comparing robot designs requires partial orders on $\dprop$ and $\sensing$.}

\begin{definition}[Poset of dynamical properties and sensing capabilities]
\label{def:poset_dyn_prop_sens_cap}
    The \F{dynamical properties} and \F{sensing capabilities} are ordered according to the componentwise order, i.e.:
\begin{equation*}
    \dprop_1 \preceq \dprop_2 
    \iff d_{1,i} \preceq_{D_i} d_{2,i} \quad \forall i = 1,\ldots,w.,
\end{equation*}
\begin{equation*}
    \sensing_1 \preceq \sensing_2 
    \iff s_{1,j} \preceq_{S_j} s_{2,j} \quad \forall j = 1,\ldots,v.
\end{equation*}
\end{definition}

\rev{Least elements and null robot implementations allow flexible fleet size in a fixed co-design diagram.}

\begin{assumption}[Least elements]
\label{ass:least_element}
The \glspl{abk:poset} of dynamical and sensing properties admit least elements~$\bot_{\dprop}$ and~$\bot_{\sensing}$, respectively, such that~$\bot_{\dprop}\preceq \dprop$ and~$\bot_{\sensing}\preceq \sensing$ for all~$\dprop$ and~$\sensing$.
\rev{
When necessary, we introduce formal least elements to represent the null robot's capabilities.
}
\end{assumption}

\begin{figure}[tb]
\centering
\setlength{\abovecaptionskip}{0pt}
\includegraphics[width=0.7\columnwidth]{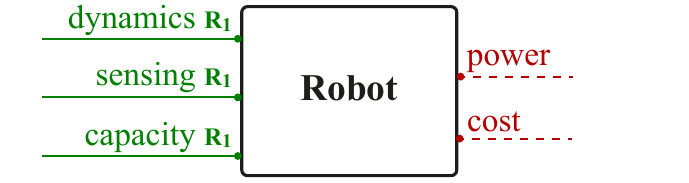}
\caption{
\rev{
Interface-level abstraction of the \textbf{Robot} \gls{acr:mdpi}. 
The robot provides dynamical properties, sensing capabilities, and energy capacity as functionalities while requiring operational power and acquisition cost as resources.
}
}
\label{fig:general_codesign_robot_mcdp}
\end{figure}

\begin{definition}[Null robot]
\label{def:null_robot}
\rev{A \emph{null robot} $\robot_\bot$ is an implementation that provides the least dynamical and sensing tuples $(\bot_{\dprop},\bot_{\sensing})$, has zero energy capacity, and requires zero resources (zero cost and zero power).
Its associated execution-time budget is $t^{\mathrm{exec}}(\robot_\bot)=0$.}
\end{definition}

\begin{lemma}[Robot MDPI]
\label[lemma]{lem:robot_mdpi_monotone}
The \textbf{Robot} \gls{acr:mdpi} is valid.
\end{lemma}
\begin{proof}
\rev{ 
An implementation that requires $\R{\text{cost}}_1$ and $\R{\text{power}}_1$ to provide $\dprop_1$ and $\sensing_1$ is still available under any $\R{\text{cost}}_2\geq\R{\text{cost}}_1$ and $\R{\text{power}}_2\geq\R{\text{power}}_1$, and delivers every smaller dynamical and sensing capability by \cref{def:robot_dyn_prop,def:robot_sens_prop}.
}
\end{proof}

\subsubsection{Fleet Composer MDPI}
\rev{The \textbf{Fleet Composer} \gls{acr:mdpi} shown in \Cref{fig:general_fleet_composer} interconnects~$n$ \textbf{Robot} \glspl{acr:mdpi} to synthesize a fleet.
As \F{functionalities} it provides (i) the fleet and (ii) an execution-time budget for each robot.
Its \R{resources} are the total cost (USD) and total energy (Wh) associated with realizing and operating that fleet.
}

\begin{figure}[tb]
\centering
\setlength{\abovecaptionskip}{0pt}
\includegraphics[width=0.7\columnwidth]{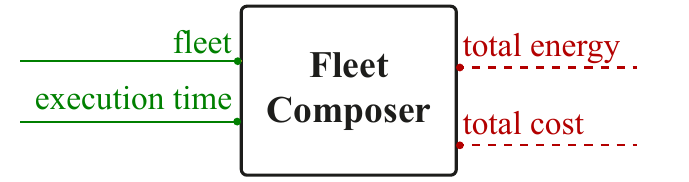}
\caption{The \textbf{Fleet Composer} \gls{acr:mdpi}.}
\label{fig:general_fleet_composer}
\end{figure}

\rev{
Crucially, the fleet representation must support comparison and co-design composition.
We thus expose only \emph{effective capabilities} to downstream planning and execution: the dynamical and sensing tuples of each robot, and an execution-time budget for how long they can be sustained.
}

\paragraph{Fleet Representation and Order}
We represent a fleet as a finite collection of capability pairs
\begin{equation*}
\fleet = \tup{\tup{\dprop_1,\sensing_1},\ldots,\tup{\dprop_n,\sensing_n}},
\end{equation*}
where each pair~$\tup{\dprop_i,\sensing_i}$ describes the effective capabilities of robot~$\robot_i$ as exposed to planning and execution.
The corresponding execution-time budgets are collected in~$\tbudvec = \tup{\tbud_1,\ldots,\tbud_n}\in \mathbb{R}_{\ge 0}^n.$
We order execution-time vectors componentwise using the standard order on~$\mathbb{R}_{\ge 0}$.

\begin{definition}[Partial order on fleets]
\label{def:partial_order_fleets}
Let
\begin{equation*}
\fleet_1 = \tup{\tup{\dprop_{1,1},\sensing_{1,1}},\ldots,\tup{\dprop_{1,n},\sensing_{1,n}}}
\end{equation*}
and
\begin{equation*}
\fleet_2 = \tup{\tup{\dprop_{2,1},\sensing_{2,1}},\ldots,\tup{\dprop_{2,m},\sensing_{2,m}}}
\end{equation*}
be two fleets (possibly of different cardinalities).
We write~$\fleet_1 \preceq_F \fleet_2$ if and only if there exists an injective map
\begin{equation*}
    \varphi : \{1,\ldots,n\} \hookrightarrow \{1,\ldots,m\}
\end{equation*}
such that
\begin{equation*}
    \dprop_{1,i} \preceq \dprop_{2,\varphi(i)}
    \quad\text{and}\quad
    \sensing_{1,i} \preceq \sensing_{2,\varphi(i)}
    \qquad \forall i\in\{1,\ldots,n\}.
\end{equation*}
In words, every robot in~$\fleet_1$ is dominated (in dynamical and sensing capabilities) by a \emph{distinct} robot in~$\fleet_2$.
\end{definition}

\begin{remark}[\rev{Re-indexing of robots and antisymmetry}]
\label{rem:fleets_unlabeled}
\rev{
Indices for robots in a fleet are just for book-keeping:
re-indexing them does not change the fleet.
Formally, the relation~$\preceq_F$ in \Cref{def:partial_order_fleets} becomes antisymmetric and a partial order if we identify all the fleets that are the same with re-indexing.
}
\end{remark}

\paragraph{Fleet Composer Interconnection}
The detailed \textbf{Fleet Composer} structure is shown in \Cref{fig:general_fleet_composer_detailed}.
Each robot provides its power requirement and energy capacity.
The \textbf{Fleet Composer} multiplies each robot's power by its execution-time budget to obtain the energy required to operate that robot for the allowed time.
A feedback constraint enforces feasibility by requiring the provided capacity to dominate the required energy.
Summing acquisition costs yields total cost; summing operating energies yields total energy.

\begin{figure}[tb]
\centering
\setlength{\abovecaptionskip}{0pt}
\includegraphics[width=\linewidth]{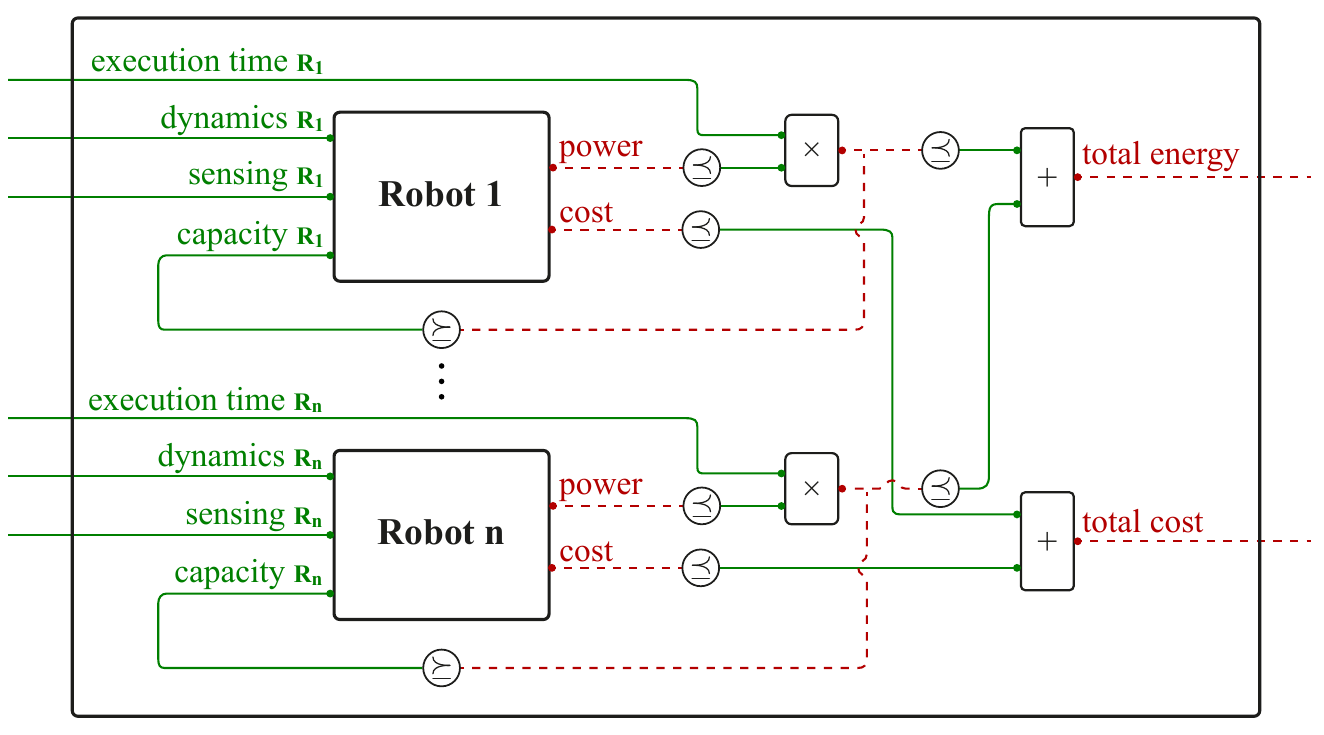}
\caption{Interconnection of multiple \textbf{Robot}
\gls{acr:mdpi}s to form the \textbf{Fleet Composer}
\gls{acr:mdpi}.}
\label{fig:general_fleet_composer_detailed}
\end{figure}

\rev{
The construction makes explicit the feedback between robot-level energy capacity and downstream execution choices: planning and execution determine how long each robot need to operate, hence the execution-time budgets and then the required energy capacity.
}

\begin{lemma}[Fleet Composer MDPI]
\label[lemma]{lem:fleet_composer_monotone}
The \textbf{Fleet Composer} \gls{acr:mdpi} is valid.
\end{lemma}
\begin{proof}
The \textbf{Fleet Composer} is obtained by composing \textbf{Robot} \glspl{acr:mdpi} through series/parallel connections and a feedback loop.
All intermediate arithmetic operations used in the interconnection (addition and multiplication on~$\mathbb{R}_{\ge 0}$) are monotone maps.
Therefore, by closure of \rev{valid} \glspl{acr:mdpi} under these compositions~\cite{zardini2023dissertation}, the resulting \textbf{Fleet Composer} \rev{\gls{acr:mdpi} is valid}.
\end{proof}

\begin{remark}[Customizing co-design models]
\label{rem:co_desigin_customization}
\rev{
We emphasize that the formal co-design model is not limited to cost and energy, while they illustrate two distinct resource types.
Cost is determined locally by the selected fleet, whereas energy couples local robot properties such as power with downstream execution information such as the required duration, and a more detailed energy model may add trajectory-dependent terms computed by the \textbf{Executor}.
A more practical model may introduce other resources or refine these, as long as the interfaces stay monotone.
}
\end{remark}

\begin{remark}[Fixed-length representation via null robots]
\label{rem:fixed_fleet_length}
\rev{
The fixed number $n$ of robot components in \Cref{fig:general_fleet_composer_detailed} is the \emph{maximum} fleet size the diagram supports, \emph{not} a count of active robots.
An inactive component takes the null robot $\robot_\bot$ of \Cref{def:null_robot}, which has zero cost and power, provides zero capacity with $\tup{\bot_{\dprop},\bot_{\sensing}}$ and $t^{\mathrm{exec}}(\robot_\bot)=0$, and thus represents an absent robot.
}
\end{remark}

\subsection{The Planner Design Problem}
\label{subsec:general_planner_design}
We next describe the \textbf{Planner} \gls{acr:mdpi} shown in \Cref{fig:general_codesign_planner}.
This design problem abstracts the choice of planning software that, given a fleet and an environment, generates waypoint assignments (\Cref{def:planner}).

\begin{figure}[tb]
\centering
\setlength{\abovecaptionskip}{0pt}
\includegraphics[width=0.7\columnwidth]{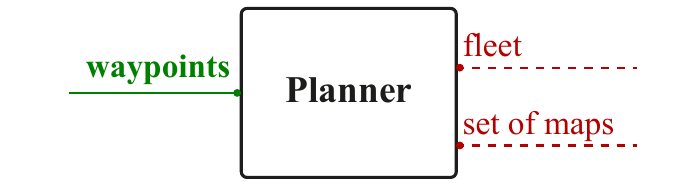}
\caption{The \textbf{Planner} \gls{acr:mdpi}.}
\label{fig:general_codesign_planner}
\end{figure}

Within co-design, the \textbf{Planner} \gls{acr:mdpi} \R{requires} a fleet and a set of maps and \F{provides} a set of waypoint collections.
Intuitively, providing the planner with a more capable fleet or a more permissive environment cannot reduce the set of waypoint options available.

\paragraph{Waypoint Collections}
To reason about planning outputs within co-design, we order waypoint collections by an injective subsequence relation that preserves waypoint order.
This captures the idea that a planner that can generate ``more'' waypoint options (e.g., by covering larger regions or by offering additional waypoints that can be ignored downstream) provides a larger functionality.

\begin{definition}[Poset of waypoint collections]
\label{def:poset_collection_waypoints}
Let
\begin{equation*}
\collwpoints^{(1)} = \prod_{i=1}^{n} \wpoints^{(1)}_i
\quad\text{and}\quad
\collwpoints^{(2)} = \prod_{j=1}^{m} \wpoints^{(2)}_j
\end{equation*}
be two waypoint collections (with possibly different cardinalities), where each $\wpoints_i = (w_{i,1},\dots,w_{i,p_i})$ is an ordered waypoint sequence.
We define a partial order~$\preceq_{\collwpoints}$ by declaring that $\collwpoints^{(1)} \preceq_{\collwpoints} \collwpoints^{(2)}$
if and only if there exists an injective map $\varphi : \{1,\ldots,n\}\hookrightarrow \{1,\ldots,m\}$ such that, for every \(i\in\{1,\ldots,n\}\), the sequence \(\wpoints_i^{(1)}\) is an order-preserving subsequence of \(\wpoints_{\varphi(i)}^{(2)}\), i.e., there exists a strictly increasing map
\begin{equation*}
\sigma_i : \{1,\ldots,p_i^{(1)}\}\to\{1,\ldots,p_{\varphi(i)}^{(2)}\}
\end{equation*}
satisfying
\begin{equation*}
w^{(1)}_{i,k} = w^{(2)}_{\varphi(i),\,\sigma_i(k)}
\qquad
\forall k\in\{1,\ldots,p_i^{(1)}\}.
\end{equation*}
\end{definition}

\rev{
From the compositional perspective, this subsequence order assumes that the executor is compatible with waypoint refinement: inserting additional intermediate waypoints into a waypoint sequence cannot invalidate the feasibility of executing the original sequence, since the executor may realize the same motion by traversing the matched subsequence and ignoring additional inserted waypoints.
}

\paragraph{Sets of Maps}
Recall from \Cref{def:setofmaps} that a set of maps is a finite set~$\setofmaps=\{\map_1,\ldots,\map_p\}$.
We order sets of maps by injective inclusion.

\begin{definition}[Poset of sets of maps]
\label{def:poset_set_maps}
Let~$\setofmaps^{(1)} = \{ \map_{1,1},\ldots,\map_{1,n} \}$ and~$\setofmaps^{(2)} = \{ \map_{2,1},\ldots,\map_{2,m} \}$ be two sets of maps.
We write~$\setofmaps^{(1)} \preceq_{\setofmaps} \setofmaps^{(2)}$ if and only if there exists an injective map~$\varphi : \{1,\ldots,n\}\hookrightarrow\{1,\ldots,m\}$ such that
\begin{equation*}
    \map_{1,i} \subseteq \map_{2,\varphi(i)}
    \qquad \forall i\in\{1,\ldots,n\}.
\end{equation*}
\end{definition}

\paragraph{Sets of Waypoint Collections}
A planner provides a waypoint collection for each map in a set of maps.

\begin{definition}[Poset of sets of waypoint collections]
\label{def:poset_set_waypoints}
Let~$\setcollwpoints^{(1)} = \{ \collwpoints_{1,i}\}_{i=1,\cdots,n}$ and~$\setcollwpoints^{(2)} = \{ \collwpoints_{2,m}\}_{i=1,\cdots,m}$ be two sets of waypoint collections.
We say~$\setcollwpoints^{(1)} \preceq_{\setcollwpoints} \setcollwpoints^{(2)}$ if and only if there exists an injective map~$\varphi : \{1,\ldots,n\}\hookrightarrow\{1,\ldots,m\}$ such that
\begin{equation*}
    \collwpoints_{1,i} \preceq_{\collwpoints} \collwpoints_{2,\varphi(i)}
    \qquad \forall i\in\{1,\ldots,n\},
\end{equation*}
where $\preceq_{\collwpoints}$ is the order from \Cref{def:poset_collection_waypoints}.
\end{definition}

\begin{lemma}[Planner MDPI]
\label[lemma]{lem:planner_monotonicity}
The \textbf{Planner} \gls{acr:mdpi} is valid.
\end{lemma}

\begin{proof}
\rev{
We emphasize that an implementation for \textbf{Planner} may contain not only an algorithm, but also more information such as required number and properties of robots.
Suppose an implementation requires fleet $\fleet_1$ and map collection $\setofmaps^{(1)}$ to provide the waypoint collection $\setcollwpoints^{(1)}$.
Under a better fleet $\fleet_2$ and a larger map collection $\setofmaps^{(2)}$ the same implementation picks the sub-fleet and the sub-collection of maps that \Cref{def:partial_order_fleets,def:poset_set_maps} match to $\fleet_1$ and $\setofmaps^{(1)}$, produces $\setcollwpoints^{(1)}$, and delivers any $\setcollwpoints^{(2)}\preceq_{\setcollwpoints}\setcollwpoints^{(1)}$ by picking the sub-collection and smaller waypoints by \Cref{def:poset_set_waypoints}.
}
\end{proof}

\subsection{The Executor Design Problem}
\label{subsec:general_executor_design}
The \textbf{Executor} \gls{acr:mdpi} (\Cref{fig:general_codesign_executor}) instantiates the executor interface from \Cref{def:executor}.
It \R{requires} a fleet, a set of waypoint collections, and an execution-time budget vector, and \F{provides} a set of dynamically feasible trajectory collections.

\begin{figure}[tb]
\centering
\setlength{\abovecaptionskip}{0pt}
\includegraphics[width=0.7\columnwidth]{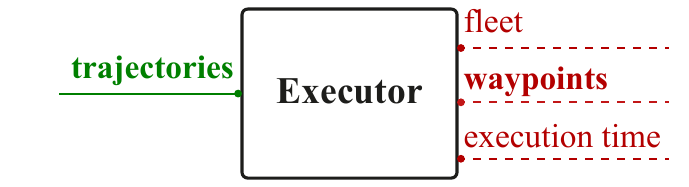}
\caption{The \textbf{Executor} \gls{acr:mdpi}.}
\label{fig:general_codesign_executor}
\end{figure}

To use trajectories within co-design, we equip them with a partial order that captures the idea that a trajectory can be safely extended in time (e.g., by waiting) and can be improved by stronger sensing.

\begin{definition}[Poset of trajectories]
\label{def:poset_trajectories}
Let~$\traj_1:[0,T_1]\to\map$ and~$\traj_2:[0,T_2]\to\map$ be two trajectories as in \Cref{def:trajectory}.
We associate with each trajectory a (time-invariant) sensing tuple~$\sensing_{\traj_i}\in \sensing$ describing the sensing capabilities of the robot executing it.
We define\rev{~$\preceq_\traj$ by:} $\traj_1\preceq_\traj \traj_2$ if and only if
\begin{enumerate}
    \item \textbf{Prefix dominance in time and space:} $T_1\le T_2$ and $\traj_1(t)=\traj_2(t)$ for all $t\in[0,T_1]$;
    \item \textbf{Sensing dominance:} $\sensing_{\traj_1}\preceq \sensing_{\traj_2}$, using the order from \Cref{def:poset_dyn_prop_sens_cap}.
\end{enumerate}
\end{definition}

\begin{definition}[Poset of trajectory collections]
\label{def:poset-trajectory-collections}
Let~$\colltraj_1 = \{\traj^{(1)}_1,\ldots,\traj^{(1)}_n\}$ and~$\colltraj_2 = \{\traj^{(2)}_1,\ldots,\traj^{(2)}_m\}$ be two collections of trajectories.
We write~$\colltraj_1 \preceq_{\colltraj} \colltraj_2$ if and only if there exists an injective map~$\varphi:\{1,\ldots,n\}\hookrightarrow\{1,\ldots,m\}$ such that
\begin{equation*}
    \traj^{(1)}_i \preceq_\traj \traj^{(2)}_{\varphi(i)}
    \qquad \forall i\in\{1,\ldots,n\}.
\end{equation*}
\end{definition}

\begin{definition}[Poset of sets of trajectory collections]
\label{def:poset_set_trajectory_collections}
Let $\setcolltraj^{(1)} = \{ \colltraj_{1,1},\ldots,\colltraj_{1,n} \}$ and $\setcolltraj^{(2)} = \{ \colltraj_{2,1},\ldots,\colltraj_{2,m} \}$ be two sets of trajectory collections.
We define $\setcolltraj^{(1)} \preceq_{\setcolltraj} \setcolltraj^{(2)}$ if and only if there exists an injective map $\varphi : \{1,\ldots,n\}\hookrightarrow \{1,\ldots,m\}$ such that
\begin{equation*}
    \colltraj_{1,i} \preceq_{\colltraj} \colltraj_{2,\varphi(i)}
    \qquad \forall i\in\{1,\ldots,n\},
\end{equation*}
where $\preceq_{\colltraj}$ is from \Cref{def:poset-trajectory-collections}.
\end{definition}

\rev{The waypoint and trajectory orders are deliberately conservative.
Waypoint sequences obtained by inserting or removing waypoints within the same plan, as well as temporal prefixes and extensions of the same trajectory, are comparable. 
By contrast, independently generated plans and trajectories will generally be incomparable.
Consequently, the antichains may be large and prune little, with the optimality guarantee preserved.
This limited comparability is specific to the \textbf{Planner} and \textbf{Executor} interfaces.
At robot-design and fleet-composition interfaces, the capability--resource orders allow dominated alternatives to be pruned locally before they are propagated through the remaining co-design problem.}

Recall~$\tbudvec = \tup{\tbud_1,\ldots,\tbud_n}\in \mathbb{R}_{\ge 0}^n$, the execution-time budget vector, ordered componentwise.

\begin{lemma}[Executor MDPI]
\label[lemma]{lem:executor_monotonicity}
The \textbf{Executor} \gls{acr:mdpi} is valid.
\end{lemma}

\begin{proof}
\rev{
As for the \textbf{Planner}, an implementation for \textbf{Executor} may contain more than a tracking algorithm: it may also specify which robots it drives, how to select the waypoints it visits, and how long each robot runs.
Then suppose one implementation requires fleet $\fleet_1$, set of waypoint collections $\setcollwpoints^{(1)}$, and execution-time budgets $\tbudvec^{(1)}$, to provide the set of trajectory collections $\setcolltraj^{(1)}$.
With a better fleet $\fleet_2$, a better set of waypoint collections $\setcollwpoints^{(2)}$, and larger execution-time budgets $\tbudvec^{(2)}$, the same implementation can still be chosen:
it drives the sub-fleet of $\fleet_2$ that \cref{def:partial_order_fleets} matches to $\fleet_1$,
it visits the waypoint subsequences of $\setcollwpoints^{(2)}$ that \cref{def:poset_set_waypoints,def:poset_collection_waypoints} match to $\setcollwpoints^{(1)}$ and ignores the waypoints inserted between them,
and it keeps the same trajectory durations, which already fit within $\tbudvec^{(1)}$ and therefore within the larger $\tbudvec^{(2)}$.
The same set of trajectory collections $\setcolltraj^{(1)}$ is thus produced.
The implementation also delivers any smaller set of trajectory collections $\setcolltraj \preceq_{\setcolltraj} \setcolltraj^{(1)}$, by dropping the collections that are not matched, stopping each remaining robot at the prefix duration, and reporting the weaker sensing tuple, as required by \cref{def:poset_trajectories,def:poset-trajectory-collections,def:poset_set_trajectory_collections}.
}
\end{proof}

\begin{remark}[Interpretation of execution-time budgets]
\label{rem:exec_time_interpretation}
\rev{
The budget $\tbud_i$ bounds the execution horizon of \emph{every} trajectory for robot $\robot_i$ across the set of trajectory collections provided by \textbf{Executor}, consistently with \Cref{def:executor,def:poset_trajectories}.
}
\end{remark}

\subsection{The Evaluator Design Problem}
\label{subsec:general_evaluator_design}

The \textbf{Evaluator} \gls{acr:mdpi} (\Cref{fig:general_codesign_evaluator}) assesses the performance of a fleet executing a task based on the trajectories produced by the \textbf{Executor}.
It \R{requires} a fleet and a set of trajectory collections and \F{provides} (i) a set of maps on which the system can operate and (ii) a performance vector capturing task-relevant metrics (\Cref{def:evaluator,def:metric}).

\begin{figure}[tb]
\centering
\setlength{\abovecaptionskip}{0pt}
\includegraphics[width=0.7\columnwidth]{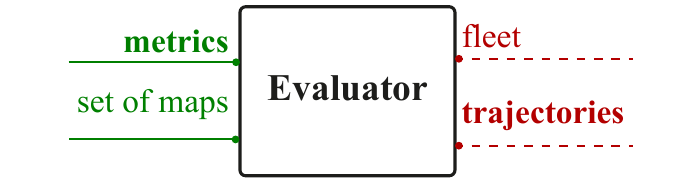}
\caption{The \textbf{Evaluator} \gls{acr:mdpi}.}
\label{fig:general_codesign_evaluator}
\end{figure}

\begin{assumption}[Least elements in performance spaces]
\label{ass:least_element_metrics}
Each performance space~$\tup{\perfspace_i,\preceq_{\perfspace_i}}$ admits a least element $\bot_i$.
This allows task profiles to selectively ignore metrics by setting their thresholds to $\bot_i$.
For example, for coverage-percentage metrics, one can take $\bot_i=0\%$.
\end{assumption}

\begin{definition}[Poset of performance metric vectors]
\label{def:poset_perf_metric_space}
Let~$\perfvecspace \coloneq \perfspace_1 \times \cdots \times \perfspace_q$ be the space of performance metric-value vectors.
We equip~$\perfvecspace$ with the product order: for~$\mathbf{m}_1=\tup{m_{1,1},\ldots,m_{1,q}}$ and $\mathbf{m}_2=\tup{m_{2,1},\ldots,m_{2,q}}$,
\begin{equation*}
\mathbf{m}_1 \preceq_{\perfvecspace} \mathbf{m}_2
\iff
m_{1,i} \preceq_{\perfspace_i} m_{2,i}\quad \forall i\in\{1,\ldots,q\}.
\end{equation*}
\end{definition}

Putting all of the above together, we are able to prove the monotonicity of the \textbf{Evaluator}.

\begin{lemma}[Evaluator MDPI]
\label[lemma]{lem:evaluator_monotonicity}
Assume each performance metric $m_i$ is monotone in the fleet and in the trajectory collection (\Cref{def:metric}).
Then the \textbf{Evaluator} \gls{acr:mdpi} is valid.
\end{lemma}

\begin{proof}
\rev{
An implementation for \textbf{Evaluator} fixes which fleet and which set of trajectory collections it inspects, together with the auxiliary map information its metrics read (\cref{rem:map_plus_internal}).
Then suppose one implementation requires fleet $\fleet_1$ and set of trajectory collections $\setcolltraj^{(1)}$, to provide the certified set of maps $\setofmaps^{(1)}$ and the performance metric-value vector $\mathbf{m}_1$.
With a better fleet $\fleet_2$ and a larger set of trajectory collections $\setcolltraj^{(2)}$, the same implementation can still be applied, by choosing the sub-fleet of $\fleet_2$ that \cref{def:partial_order_fleets} matches to $\fleet_1$ and the subset of trajectory collections of $\setcolltraj^{(2)}$ that \cref{def:poset_set_trajectory_collections} matches to $\setcolltraj^{(1)}$.
Every metric is monotone in the fleet and in the trajectory collection, so the values read on the matched arguments dominate $\mathbf{m}_1$ componentwise in the order of \cref{def:poset_perf_metric_space}.
The same certified set of maps $\setofmaps^{(1)}$ and performance metric-value vector $\mathbf{m}_1$ are thus provided.
The implementation also can certify any smaller performance metric-value vector with any smaller set of maps, because operating on fewer maps, each contained in a map it has already certified, with smaller (worse) metric values, is a easier task.
}
\end{proof}

\begin{remark}[Internal handling of auxiliary map information]
\label{rem:map_plus_internal}
Auxiliary map information is not exposed as an explicit wire in the co-design diagram.
Instead, it is treated as an internal parameter of the performance metric definitions.
This design choice keeps the \textbf{Evaluator} interface minimal while allowing individual metrics to depend on additional task- or environment-specific information as needed. 
An example of auxiliary information is the locations of targets to be detected, which are not known a priori to the \textbf{Planner} or \textbf{Executor}.
\end{remark}

\begin{remark}[Scope of admissible performance metrics]
\label{rem:monotone_metrics}
\rev{
A performance metric does not need to be monotone with the conventional order on $\mathbb{R}$, but only with the selected orders on the trajectory and performance spaces.
A metric whose smaller values are better can be accommodated by taking the opposite order on its codomain, as long as it makes it monotone for the trajectory order.
Saturation or diminishing returns violate nothing as long as extending or improving a trajectory never worsens the metric.
}

\rev{By contrast, a metric that can both improve and deteriorate along comparable trajectories is genuinely non-monotone.
Such an objective sometimes admits an equivalent reformulation as several monotone functionalities, resources, or constraints, depending on the metric and the interface orders, while a metric with no equivalent monotone representation falls outside the scope of the framework.}
\end{remark}

\subsection{Putting it All Together: The Global Co-Design Problem}
\label{subsec:global_codesign_problem}

\rev{\Cref{fig:general_codesign_model} shows the global co-design problem formed by interconnecting the \glspl{acr:mdpi} from \Cref{subsec:general_fleet_design,subsec:general_planner_design,subsec:general_executor_design,subsec:general_evaluator_design}.
The \textbf{Fleet Composer} requires \R{total cost} and \R{total energy} and provides a \F{fleet}, represented by effective dynamical and sensing capabilities, and per-robot \F{execution-time budgets}.
The \textbf{Planner} requires the \R{fleet} and a \R{set of maps} and provides a \F{set of waypoint collections}.
The \textbf{Executor} requires the \R{fleet}, \R{set of waypoint collections}, and \R{execution-time budgets} and provides a dynamically feasible \F{set of trajectory collections}.
Finally, the \textbf{Evaluator} requires the \R{fleet} and \R{set of trajectory collections} and provides a certified \F{set of maps} and \F{performance metrics}, which encode the task profile (\Cref{def:task_profile}).}

\begin{lemma}[Well-posedness of the global co-design problem]
\label{prop:global_mdpi}
The interconnection in \Cref{fig:general_codesign_model} defines a \rev{valid} \gls{acr:mdpi}.
\end{lemma}
\begin{proof}
Each block is modeled as a \rev{valid} \gls{acr:mdpi}
(\Cref{lem:robot_mdpi_monotone,lem:fleet_composer_monotone,lem:planner_monotonicity,lem:executor_monotonicity,lem:evaluator_monotonicity}).
By the closure properties of monotone co-design under series, parallel, and feedback composition~\cite{zardini2023dissertation}, the resulting interconnected multigraph is again a \rev{valid} \gls{acr:mdpi}.
\end{proof}

\rev{The interconnected pipeline forms a single composite \gls{acr:mdpi} that requires \R{total cost} and \R{total energy} and provides \F{task-level performance metrics}.
By querying this composite \gls{acr:mdpi}, one can directly identify Pareto-optimal heterogeneous multi-robot systems.
For instance, \texttt{FixFunMinRes} query identifies Pareto-optimal systems satisfying a specified task profile while minimizing these resources.
Notably, we apply \Cref{def:fun_flip} to make the metric \R{mission completion time} a resource to optimize over.
Each solution specifies both the system-level trade-off and its realizing fleet composition, robot designs, planner, and executor.}

\rev{Each \gls{acr:mdpi} can also be queried independently.
For example, applying \texttt{FixFunMinRes} to the \textbf{Fleet Composer} identifies Pareto-optimal fleets satisfying specified dynamical and sensing requirements while minimizing total cost and energy, whereas \texttt{FixResMaxFun} identifies the best capabilities achievable under given cost and energy budgets.
Analogous queries enable targeted analysis of robot design, planning, or execution without evaluating the complete pipeline.}

\rev{Beyond formal correctness, this formulation highlights a key advantage of co-design: its implementation-agnostic, poset-valued interfaces allow planners, executors, controllers, and evaluators to be replaced or extended, while abstract performance metrics and task profiles accommodate different missions without changing the surrounding structure.}

\rev{A further key advantage of co-design is its explicit representation of feedback among design layers: battery, sensing, and actuation choices constrain planning and execution through robot dynamics and energy availability.
Conversely, execution-time and trajectory requirements constrain robot and fleet design through energy capacity, while the \textbf{Evaluator} feeds achievable maps and performance levels back to the \textbf{Planner}.
Monotone co-design captures and exploits these interdependencies in a principled and computationally tractable manner.}

\subsection{Computational Advantage of Compositional Co-Design}
\label{subsec:computational_co_design}

A central motivation for the proposed formulation is computational.
Designing a heterogeneous multi-robot system requires simultaneous choices over robot hardware, fleet composition, planning, execution, and task-level evaluation.
If treated monolithically, these choices induce a very large Cartesian product of candidate system designs, and evaluating each candidate requires running the full pipeline.
The compositional co-design formulation does not remove the need to model or simulate components, but it changes \emph{how} this computation is organized: instead of enumerating all system-level combinations, it reasons locally through monotone interfaces and propagates only Pareto-relevant information across components.

We first quantify the size of the monolithic design space and then explain why the co-design formulation is computationally advantageous.

\subsubsection{Size of the Monolithic Design Space}
Suppose the design of each robot $i$ is specified by $n_i$ modules, and each module $j$ admits $m_{i,j}$ candidate implementations.
The number of designs for robot $i$ is then \rev{$C_i = \prod_{j=1}^{n_i} m_{i,j}$,}
resulting in $\prod_{i}^{N}C_i$ design options for a heterogeneous fleet with $N$ agents.
Restricting the heterogeneity level, namely the number of robot types, makes the design space much smaller.
Let $T$ denote the number of robot types, $N_{\max}$ the maximum number of robots allowed per type.
For a fixed type $i$, one may either deploy no robot of that type, or deploy $k\in\{1,\dots,N_{\max}\}$ robots of that type together with one of the $C_i$ robot designs.
Hence, the number of possible sub-fleet choices for type $i$ is \rev{$1 + N_{\max} C_i$}.
Assuming independence across robot types, the total number of heterogeneous fleet designs is therefore
\begin{equation}
\label{eq:num_fleet_designs}
\prod_{i=1}^{T} \left( 1 + N_{\max} C_i \right) = \prod_{i=1}^{T}
\left( 1 + N_{\max} \prod_{j=1}^{n_i} m_{i,j} \right).
\end{equation}
Suppose $P$ is the number of planner implementations, and $E$ is the number of executor implementations, then including those choices yields the total number of system-level design candidates
\begin{equation}
\label{eq:num_system_designs}
\begin{aligned}
N_{\mathrm{mono}}
&=
P E \prod_{i=1}^{T}
\left( 1 + N_{\max} \prod_{j=1}^{n_i} m_{i,j} \right).
\end{aligned}
\end{equation}
Equation~\eqref{eq:num_system_designs} makes explicit the combinatorial nature of the problem: the search space is multiplicative across robot types and across the hardware/software design options within each type.
\rev{
Evaluating one candidate runs the complete pipeline of planning, trajectory generation, and task evaluation, so a monolithic brute-force search costs \rev{$\mathcal{O}\!\left(N_{\mathrm{mono}} \cdot c_{\mathrm{full}}\right)$} with $c_{\mathrm{full}}$ the cost of one full system-level evaluation.
A continuous module choice is resolved by discretization or by repeated derivative-free evaluations of the full pipeline, which replaces $m_{i,j}$ by the number of evaluations used.
The key point is unchanged: each additional continuous design variable increases the number of full-pipeline evaluations required by a monolithic approach.
}

\paragraph*{Concrete size in our case studies}
In \rev{\Cref{sec:case-studies}}, we consider $T=3$ robot types (medium aerial, large aerial, and ground), with at most $N_{\max}=3$ robots per type, $P=3$ planners, and $E=1$ executor.
Each robot type is defined by three modules: actuation, sensing, and battery.
The actuation and sensing modules admit $2$ and $1$ choices, respectively.
For the battery, we consider $8$ battery technologies, each parameterized by a capacity value.

To obtain a concrete estimate, we discretize the battery capacity over the interval $(0,C_{\max}]$ with resolution $r_{\mathrm{cap}}$, yielding \rev{$N_{\mathrm{bat}}(r_{\mathrm{cap}}) = \left\lfloor \frac{C_{\max}}{r_{\mathrm{cap}}}\right\rfloor$}
capacity choices per battery technology, and therefore \rev{$8\,N_{\mathrm{bat}}(r_{\mathrm{cap}})$}
battery implementations per robot type.
With $C_{\max}=400\,\mathrm{Wh}$ and $r_{\mathrm{cap}}=2\,\mathrm{Wh}$, we obtain $N_{\mathrm{bat}}=200$.
\rev{With $N_{\mathrm{bat}}=200$, each robot type admits $C_i=2\cdot1\cdot8N_{\mathrm{bat}}=3200$ designs; substituting this into \Cref{eq:num_system_designs} yields $N_{\mathrm{mono}}=3(1+3\cdot3200)^3\approx2.66\times10^{12}$ candidate system-level designs.}

Even in this deliberately modest setup, brute-force evaluation of all candidates is clearly impractical.
Importantly, this count only reflects the number of candidate designs; it does not yet account for the computational cost of evaluating each candidate through planning, execution, and simulation-based metric computation.

\subsubsection{Why Co-Design is Computationally Advantageous}
\rev{
Co-design does not make planning or simulation free or faster, but it avoids paying the full Cartesian-product cost in \eqref{eq:num_system_designs}, because robot design, fleet composition, planning, execution, and evaluation interact through the specific monotone interfaces of \Cref{fig:general_codesign_model,fig:general_fleet_composer_detailed} rather than an unstructured black box.
This structure yields two computational benefits that are illustrated in \Cref{sec:computation-baselines}.}

\paragraph*{1) Early pruning of dominated local designs}
In a monolithic search, every candidate system design is treated as a separate object and must be evaluated as such.
In contrast, co-design reasons locally at the level of components and retains only Pareto-relevant information at each interface.
Dominated local implementations are discarded as soon as they are recognized as such and therefore never participate in downstream combinations.
For example, once a robot design is dominated by another design in terms of provided capabilities and required resources, there is no need to combine it with every planner, fleet multiplicity, and task instance.

\paragraph*{2) Solving by composition rather than enumeration}
The global problem is represented as an interconnection of monotone design problems.
Series, parallel, and feedback compositions are handled directly in the co-design formalism, and feedback loops are resolved through fixed-point algorithms~\cite{zardini2023dissertation}.
As a consequence, solving a system-level query does not require explicit traversal of all elements in the Cartesian product
\eqref{eq:num_system_designs}; instead, it consists of repeated propagation and pruning of antichains through the interconnection graph.

Therefore, the computational burden is governed by:
(i) the size of the local implementation spaces,
(ii) the widths of the intermediate antichains propagated through the interfaces, and
(iii) the number of fixed-point iterations required by the feedback loops,
rather than by the full number of monolithic candidates.
\rev{Numerical values are reported in \cref{sec:computation-baselines}.}
This is precisely the structural advantage of the co-design approach: multiplicative coupling in the naive search space is replaced by local reasoning plus controlled composition.

\paragraph*{Population cost versus query cost}
It is useful to distinguish between two phases.
First, one must \emph{populate} the component-level design problems, which may still require simulation or numerical computation (e.g., evaluating planners and executors on relevant instances).
Second, one must \emph{solve} system-level co-design queries.
The proposed framework does not eliminate the first cost, but it prevents the second phase from degenerating into a brute-force search over all global combinations.
This distinction is especially important when multiple queries are asked over the same component library, since the populated component models can be reused across tasks, thresholds, or map sets.

In summary, the computational advantage of co-design lies in replacing the monolithic enumeration of
\eqref{eq:num_system_designs}
with compositional reasoning over monotone interfaces.
For heterogeneous multi-robot systems, where the naive design space is already enormous even in modest examples, this structural reduction is essential to make system-level design queries tractable.

\begin{remark}[\rev{Intellectual tractability and collaborative design}]
The benefit of the co-design formulation is not only computational.
A monolithic optimization over robot design, fleet composition, planning, execution, and task evaluation is also difficult to interpret and manipulate: all couplings are hidden inside one large black-box search problem, making it hard to understand which subsystem drives a given trade-off or infeasibility.
By exposing explicit interfaces between modules, the co-design formulation makes these dependencies legible.
Designers can reason locally about robot design, fleet design, planning, or execution, and can modify one part of the architecture, for instance, by changing a planner, adding a robot type, or introducing a new task metric, without reformulating the entire problem.
This modularity is especially valuable in collaborative settings, where different experts contribute \rev{to} different subsystems.
Accordingly, the framework improves not only computational tractability but also the intellectual tractability of heterogeneous multi-robot system design.
\end{remark}

\subsection{Applicability and Limitations}
\label{subsec:applicability_limitations}

\rev{Monotone co-design is particularly well suited when heterogeneous hardware and software choices can be decomposed into reusable subsystems coupled through partially ordered functionality--resource interfaces, including feedback.
It combines discrete catalogs, analytical models, and simulation-derived relations without requiring differentiability or convexity, while propagating Pareto-relevant alternatives and enabling model reuse across tasks.
These benefits require the exposed relations to satisfy the framework's monotonicity and order-theoretic assumptions.
Non-monotone internal behavior may sometimes be reformulated through additional interface quantities, operating regimes, or constraints, but systems without an equivalent monotone representation remain outside the scope.
Computational gains are problem-dependent because query cost depends on primitive implementation spaces, antichain widths, and fixed-point iterations, while coarse orders can limit pruning.
Co-design also does not eliminate the simulations or experiments needed to populate and validate component models.
Other optimization methods may be preferable when they can exploit problem-specific structure.
Mixed-integer programming can address finite catalogs and algebraically representable non-monotone compatibility constraints~\cite{carlone2019robot}, differentiable co-design methods can exploit gradients in continuous design and control models~\cite{ma2021diffaqua,zhao2022graph}, and constraint programming can accommodate broader discrete constraints, including certain non-monotone relations, with Pareto-optimality guarantees under suitable subsystem-consistency conditions~\cite{wilhelm2025monotone}.
The principal benefit of monotone co-design is therefore its combination of formal compositionality, explicit feedback, reusable subsystem models, and multi-objective Pareto reasoning rather than universal computational superiority.}

%% file: Sections/6_case_studies.tex
\section{Case Study: Search and Coverage}
\label{sec:case-studies}

In this section, we instantiate the co-design pipeline developed above on a realistic heterogeneous multi-robot design problem.
We focus on \emph{search-and-coverage} missions, which capture a broad class of applications including search-and-rescue, environmental monitoring, and industrial inspection.
The goal of the case studies is twofold: (i) to demonstrate end-to-end use of the proposed framework on a concrete domain, and (ii) to illustrate how changing the design space (robot components and planners) and the task profile (maps and performance metrics) changes the set of Pareto-optimal system designs, without changing the overall co-design structure.

We consider workspaces~$M \in \setofmaps$ that are rectangular two-dimensional regions to be covered.
For each workspace~$M$, we further assume a (finite) set of target locations~$\mathbf{x}_M^i \in M$, $i = 1, \dots, N_M$, representing objects whose positions are unknown to the planner and must be detected (e.g., gas leak sources in industrial inspection).
\rev{Accordingly, we use two performance metrics: the map-area fraction covered by the fleet, $m_c\in\tup{[0,1],\leq}$, and the number of targets detected with probability at least $\delta\in(0,1)$, $m_\delta\in\tup{\mathbb{N}_0,\leq}$.}
We formalize how coverage and detection are scored in \Cref{subsec:coverage_detection_models}, after introducing the available robot types and component catalogs used to populate the design space.

\subsection{Robot Design}

\subsubsection{Exposed dynamical properties}
In our case studies, \rev{each robot's dynamical properties $\dprop$ (\Cref{def:robot_dyn_prop})} comprise:
\begin{itemize}
    \item $\displaystyle v_{\max} \!\in\! (\mathbb{R}_{\geq 0}, \leq)$: maximum velocity;
    \item $\displaystyle a_{\max,\mathrm{lat}} \!\in\! (\mathbb{R}_{\geq 0}, \leq)$: maximum lateral acceleration;
    \item $\displaystyle a_{\max,\mathrm{lon}} \!\in\! (\mathbb{R}_{\geq 0}, \leq)$: maximum longitudinal acceleration;
    \item $\displaystyle \dot{\theta}_{\max} \!\in\! (\mathbb{R}_{\geq 0}, \leq)$: maximum turning rate;
    \item $\displaystyle r_{\mathrm{turn}} \!\in\! (\mathbb{R}_{\geq 0}, \geq)$: minimum turning radius;
    \item \rev{$\displaystyle P_{\text{idle}} \!\in\! (\mathbb{R}_{\geq 0}, \geq)$: idle power consumption;}
    \item $\displaystyle c_{\mathrm{vel}} \!\in\! (\mathbb{R}_{\geq 0}, \geq)$: power--velocity coefficient;
    \item $\displaystyle c_{\mathrm{acc}} \!\in\! (\mathbb{R}_{\geq 0}, \geq)$: power--acceleration coefficient;
    \item $\displaystyle p \!\in\! \{\mathrm{RS}, \mathrm{DD}\}$: path type.
\end{itemize}
Here,~$\mathrm{RS}$ denotes Reeds--Shepp and $\mathrm{DD}$ denotes differential-drive; these two types are incomparable~\cite{lavalle2006planning}.
The chosen partial orders reflect natural preferences: larger bounds on speed/acceleration/turn rate correspond to more capable motion, smaller turning radii correspond to improved maneuverability, and smaller power coefficients correspond to improved energetic efficiency.

\subsubsection{Power and energy model}
Given a trajectory with velocity profile~$v(t)$ and acceleration profile~$a(t)$, the instantaneous power consumption of a robot at time~$t$ is modeled as
\begin{equation}
\label{eq:instant_power_calculation}
P(t) = P_{\text{idle}} + c_{\text{vel}}\, v(t)^{2} + c_{\text{acc}}\, |a(t)|,
\end{equation}
where~$P_{\text{idle}}$ captures platform/sensor idling power and the remaining terms capture motion-dependent power draw.
\rev{
The parameters are specific to each robot type and actuation variant.
}

\rev{The resource trade-offs reported in the case studies depend on the chosen energy model and its parameters.
This model assumes centralized offline planning and local execution.
It does not account for onboard planning computation, continuous inter-robot communication, or online replanning.
These costs could be included in future applications through additional model terms or catalog-based energy estimates.
The co-design formulation does not require a particular energy formula: domain experts may provide alternative component models, provided that their exposed functionality--resource interfaces satisfy the required monotonicity properties.}

\begin{remark}[Energy accounting in the co-design instantiation]
Evaluating \cref{eq:instant_power_calculation} requires access to the executed trajectory (and hence to the executor output).
In our instantiation of the co-design graph, motion-dependent energy is obtained by integrating the trajectory-dependent terms in \Cref{eq:instant_power_calculation}, along the simulated execution.
This allows the \textbf{Executor} to expose the energy required by a given plan, while the \textbf{Fleet Composer} enforces feasibility by ensuring that the selected battery capacity dominates the required energy, analogously to the execution-time coupling in \Cref{subsec:general_fleet_design}.
\end{remark}

\subsubsection{Exposed sensing capabilities}
The sensing capabilities of each robot~$\sensing$ (\Cref{def:robot_sens_prop}) are characterized by the sensing radius~$r_{\mathrm{sensing}} \in (\mathbb{R}_{\geq 0}, \leq)$ and the instantaneous detection rate, which is determined by $\lambda_{\mathrm{base}} \in (\mathbb{R}_{\geq 0}, \leq)$, $\sigma_d \in (\mathbb{R}_{> 0}, \leq)$ and $\beta_v \in (\mathbb{R}_{\geq 0}, \geq)$ (described in \Cref{subsec:coverage_detection_models}).

\subsubsection{Component catalogs and robot types}
Rather than exploring the property spaces~$\dprop$ and~$\sensing$ directly, we consider a more realistic setup in which the system designer selects from a modular catalog of robot types, following the component-based \gls{acr:mdpi} model shown in \Cref{fig:general_robot_detailed}. 
This makes reasoning about component costs more faithful to practice, as trade-offs arise from discrete module choices rather than from assumed continuous spaces~$\dprop$ and~$\sensing$. Concretely, we consider three robot types: a medium-sized aerial robot, a large aerial robot, and a ground robot. 
Within each type, multiple design choices must be made: (i) the battery module (determining energy capacity), (ii) the sensing module (which determines the sensing radius and the instantaneous detection rate), and (iii) the actuation module (which determines the robot's dynamical properties). 
The available battery modules are shared across robot types.
Each battery technology is characterized by its specific energy-to-mass ratio $\rho$ (Wh/kg) and its energy-to-cost ratio $\alpha$ (Wh/USD).
The available battery technologies are listed in \Cref{tab:battery-modules}.
Each robot type offers two actuation-module variants (summarized in \Cref{tab:actuation-modules}) and one sensing module (specified in \Cref{tab:sensing-modules}).

\input{Figures/table_battery_modules}
\input{Figures/table_actuation_modules}
\input{Figures/table_sensing_modules}

\begin{figure}[tb]
\centering
\setlength{\abovecaptionskip}{0pt}
\includegraphics[width=0.85\columnwidth]{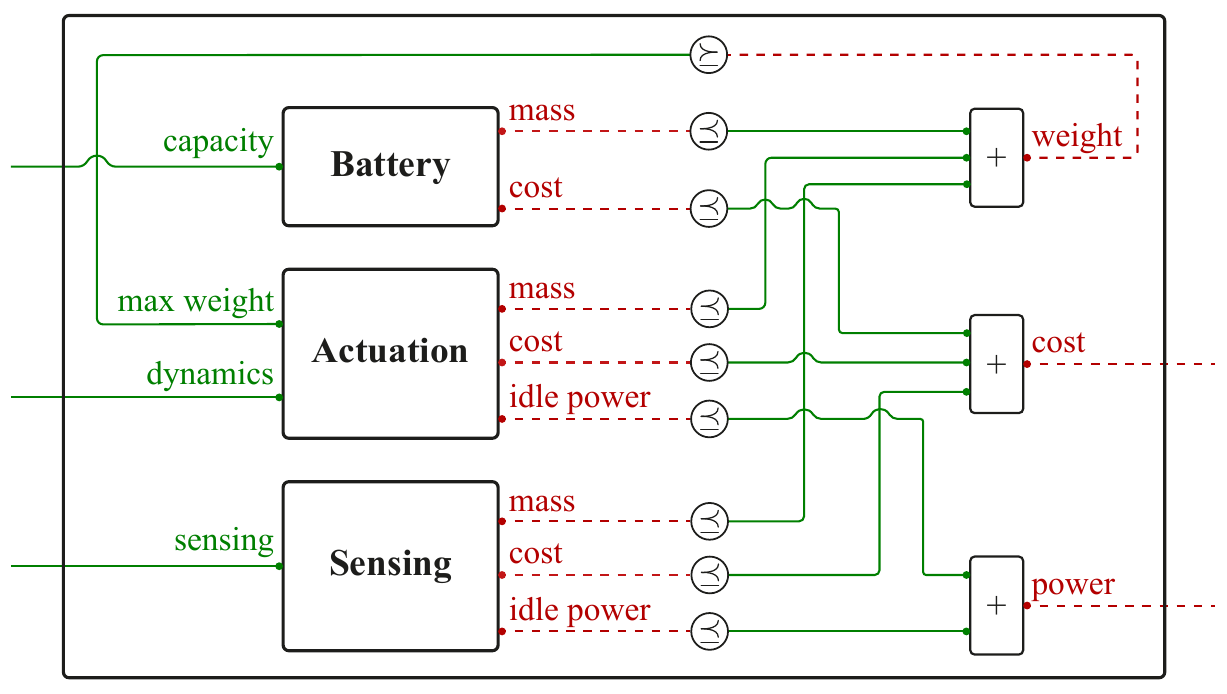}
\caption{\rev{
Component-level instantiation of the \textbf{Robot} \gls{acr:mdpi}. 
Battery, actuation, and sensing modules are composed to determine the robot’s capabilities, weight feasibility, total cost, and operational power.
}}
\label{fig:general_robot_detailed}
\end{figure}

\subsection{Coverage and Detection Models}
\label{subsec:coverage_detection_models}

A point~$\mathbf{x} \in M$ is covered by a robot at position~$\mathbf{r}$ if it lies within the robot's sensing disk, i.e.,\ $\|\mathbf{x} - \mathbf{r}\| \leq r_{\mathrm{sensing}}$.
The coverage metric~$m_c$ is then the fraction of the workspace covered by the union of sensing disks swept along all fleet trajectories.

Each robot's sensor is characterized by an instantaneous detection rate
\begin{equation}
\label{eq:detection_rate}
\lambda(t,\mathbf{x})
=
\lambda_{\mathrm{base}} \,
f\big( d(t,\mathbf{x}) \big) \,
g\big( \|v(t)\| \big),
\end{equation}
where $d(t,\mathbf{x}) = \|\mathbf{r}(t) - \mathbf{x}\|$ is the distance between the robot position~$\mathbf{r}(t)$ and a target location~$\mathbf{x}$.
The spatial decay and velocity-dependent degradation are modeled, respectively, as
\begin{align}
\label{eq:spatial_decay}
f(d) &=
\exp \big(\! - \frac{d^2}{2\sigma_d^2} \big) \,
\mathbf{1}\{ d \leq r_{\mathrm{sensing}} \}, \\
g(\|v\|) &=
\exp \big(\! - \beta_v \|v\| \big),
\end{align}
where $\sigma_d$ controls the spatial uncertainty of the sensor and $\beta_v$ captures the degradation of detection performance with increasing robot velocity.
For a given target location~$\mathbf{x}$, the cumulative exposure over a time horizon~$[0,T]$ is
\begin{equation}
\label{eq:exposure}
E(\mathbf{x})
=
\int_{0}^{T} \lambda(t,\mathbf{x}) \, dt.
\end{equation}
Modeling detection opportunities as independent micro-events and assuming that exposures add linearly across robots, the fleet-level probability of detecting a target at location~$\mathbf{x}$ is
\begin{equation}
\label{eq:fleet_hit_probability}
P_{\mathrm{hit}}(\mathbf{x})
=
1 - \exp \Big(\! - \sum_{r} E_r(\mathbf{x}) \Big),
\end{equation}
where $E_r$ is the cumulative exposure accumulated by robot~$r$ over its trajectory.
The detection metric~$m_\delta$ then counts the number of targets~$\mathbf{x}_M^i$ for which~$P_{\mathrm{hit}}(\mathbf{x}_M^i) \geq \delta$.

\rev{
Mission completion time (\Cref{def:mission_completion_time}) measures how quickly the fleet meets task requirements specified through the coverage metric $m_c$ and detection metric $m_\delta$.
The \textbf{Evaluator} provides this metric through the \F{metrics} functionality wire, and the construction in \Cref{def:fun_flip} moves it to the resource side for joint minimization with total cost and energy.
}

\begin{figure}[tb]
\centering
\setlength{\abovecaptionskip}{0pt}
\includegraphics[width=0.95\linewidth]{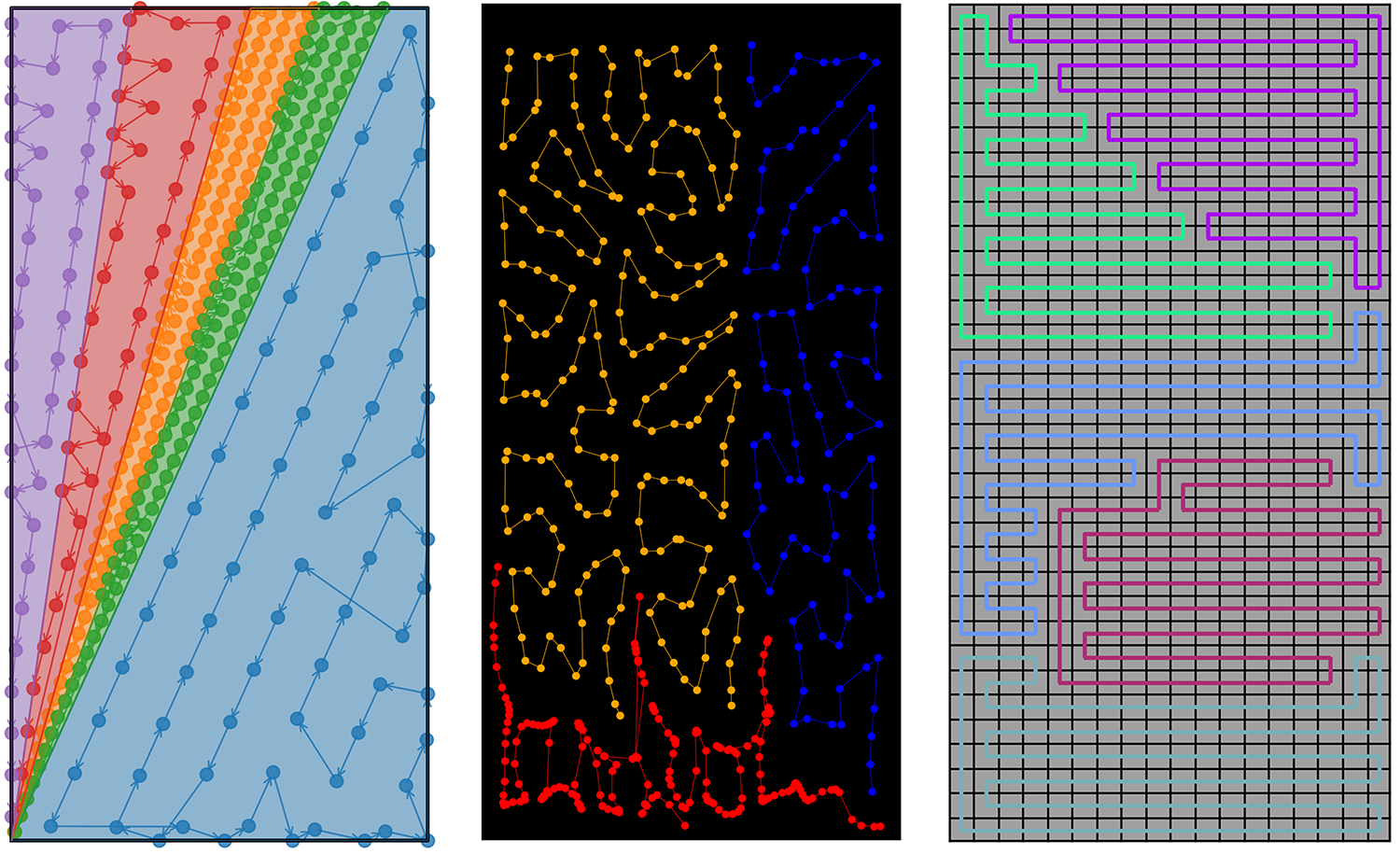}
\caption{Example waypoint solutions generated by three planners for a rectangular map of dimensions
$400\,\mathrm{m} \times 800\,\mathrm{m}$.
From left to right: \gls{acr:agd}, \gls{acr:mrta}, and \gls{acr:darp}.}
\label{fig:example_result_planner}
\end{figure}

\subsection{Planning Algorithms}
To populate the \textbf{Planner} MDPI with concrete implementations, we adapt three coverage planning algorithms from the literature: \gls{acr:agd}~\cite{kazemdehbashi2024adaptive}, \gls{acr:mrta}~\cite{hu2022large}, and \gls{acr:darp}~\cite{kapoutsis2017darp}.
All three planners output per-robot waypoint sequences, but differ substantially in how they decompose the workspace and allocate subregions to robots, which leads to qualitatively different waypoint structures (see \Cref{fig:example_result_planner}).

\paragraph{AGD planner}
Given a robot with a known sensing radius and a designated subregion, the \gls{acr:agd} algorithm~\cite{kazemdehbashi2024adaptive} generates waypoints by adaptively tiling the area with grid cells fitted to the robot's sensing footprint, minimizing the number of waypoints needed to guarantee full coverage.
To extend the method to heterogeneous multi-robot fleets, we partition the workspace into exclusive subregions using the anchored area decomposition of Hert and Lumelsky~\cite{hert1998polygon}. 
Each robot~$\robot_i$ receives a fraction of the total area proportional to~$v_i \, r_{i,\mathrm{sensing}}$, ensuring that faster robots with larger sensing footprints are assigned commensurately larger regions.
The anchored area decomposition then partitions the workspace into exclusive subregions according to these weights, after which the \gls{acr:agd} algorithm is applied independently to each robot--subregion pair. 
To determine the traversal order of the generated waypoints within each subregion, we apply the Christofides heuristic for the metric traveling salesman problem~\cite{christofides2022worst}.

\paragraph{MRTA planner}
Unlike \gls{acr:agd}, the \gls{acr:mrta} framework of Hu et al.~\cite{hu2022large} natively supports heterogeneous fleets. 
It decomposes the environment into subregions and allocates robot teams based on predicted coverage performance; each team then independently generates trajectories by minimizing an ergodic metric to achieve spatially uniform coverage. 
In our integration, we adopt only the subregion assignment and ergodic planning components, simplifying the original model by assuming uniform terrain across the workspace. 
Each robot's traversal dynamics are therefore taken directly from our actuation catalog (\cref{tab:actuation-modules}) rather than from the terrain-dependent velocity model of~\cite{hu2022large}.

\paragraph{DARP planner}
The \gls{acr:darp} algorithm~\cite{kapoutsis2017darp} computes an equitable, connected, and collision-free partition of a discretized workspace by iteratively assigning grid cells to robots until a balanced division is achieved. 
Since \gls{acr:darp} is designed for homogeneous teams, we adapt it to our heterogeneous setting by setting the grid resolution to the sensing footprint of the smallest robot in the fleet, ensuring that every robot can cover any cell assigned to it. 
The standard \gls{acr:darp} assignment procedure is then applied to obtain non-overlapping subregions, which serve as the basis for per-robot trajectory generation.

\Cref{fig:example_result_planner} shows representative waypoint solutions produced by the three planners on a $400\,\mathrm{m} \times 800\,\mathrm{m}$ workspace, illustrating the qualitative differences in spatial decomposition and waypoint structure across \gls{acr:agd}, \gls{acr:mrta}, and \gls{acr:darp}.

\subsection{Executor Details}

The \textbf{Executor} maps a fleet, a waypoint assignment, and a time budget to dynamically feasible trajectories (\cref{def:executor}).
In our case study, it is implemented through a simulation-based physics engine.
Trajectory generation respects the motion constraints induced by the selected actuation module, including bounds on speed, acceleration, and turning rate as well as the path type (e.g., Reeds--Shepp versus differential drive).
In the present case study, we use Reeds--Shepp trajectories for all robots, thereby fixing the trajectory model and isolating other design trade-offs.
For each robot, the physics engine produces time-indexed state trajectories together with velocity and acceleration profiles, which are used to evaluate the power model in \Cref{eq:instant_power_calculation} and compute the per-robot execution energy.
The resulting trajectories and energy requirements are then passed to the \textbf{Evaluator} for task-level scoring and to the \textbf{Fleet Composer} for feasibility checks on battery sizing, respectively.
\Cref{fig:example_result_executor} shows an example trajectory generated for a \gls{acr:darp} waypoint assignment on a $400\,\mathrm{m} \times 800\,\mathrm{m}$ map.

\begin{figure}[tb]
    \setlength{\abovecaptionskip}{0pt}
    \centering
    \includegraphics[width=.90\linewidth]{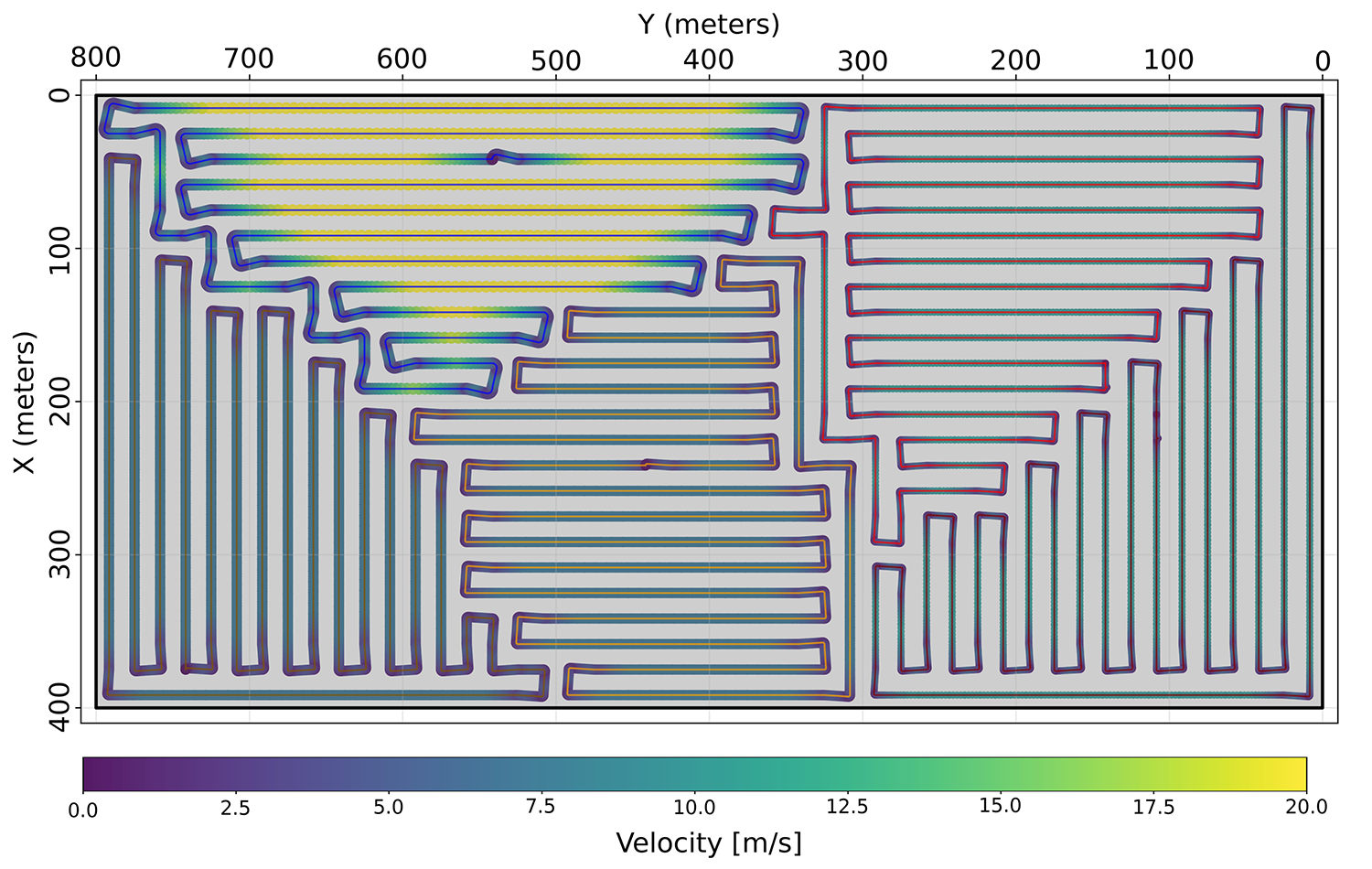}
    \caption{%
    Example trajectories generated by the \textbf{Executor} from a \gls{acr:darp} waypoint assignment on a $400\,\mathrm{m} \times 800\,\mathrm{m}$ map. 
    The resulting trajectories respect the motion constraints induced by the selected actuation modules.
    }
    \label{fig:example_result_executor}
\end{figure}

\begin{remark}[\rev{Consistency with the co-design model}]
\rev{
The search-and-coverage instantiation follows the interfaces formalized in \Cref{sec:codesign-formalization}.
The planners map fleets and maps to waypoint collections, while the \textbf{Evaluator} computes coverage and detection from executed trajectories.
The \textbf{Fleet Composer} uses one design block and an integer multiplicity for each of the three robot types, aggregating resources across types without changing these interfaces.
}
\rev{
Monotonicity concerns the \gls{acr:mdpi} feasibility relations rather than the outputs of independent algorithm runs.
A fixed module selection remains feasible under larger resource budgets.
Larger motion bounds and smaller minimum turning radii preserve previously feasible trajectories, while improved sensing preserves or increases coverage and detection.
Fleet composition preserves these properties through injective robot matching and monotone aggregation of nonnegative resources.
}
\rev{
Each generated plan is represented as an implementation with a specified fleet requirement.
If $\fleet_1 \preceq_F \fleet_2$, the plan for $\fleet_1$ remains feasible for $\fleet_2$ by assigning its waypoint sequences to the matched dominating robots, using the originally required capabilities, and leaving additional robots idle.
The \textbf{Planner} \gls{acr:mdpi} therefore retains these implementations as alternatives for dominating fleets.
This does not require \gls{acr:agd}, \gls{acr:mrta}, or \gls{acr:darp} to produce comparable waypoint collections when independently rerun on ordered fleets.
}
\rev{
Waypoint collections and trajectories use the subsequence and prefix orders of \Cref{def:poset_collection_waypoints,def:poset_trajectories}.
At the \textbf{Executor} interface, previously feasible trajectories remain available under dominating robot capabilities or larger time budgets.
Coverage and detection cannot decrease under trajectory extension or sensing improvement.
For each fixed map-specific task profile, these improvements therefore cannot increase mission completion time.
}
\end{remark}

\section{Numerical Results}
\label{sec:numerical_results}

Here, we present four case studies that progressively stress different aspects of the framework.
Case Study~I demonstrates how a standard \texttt{FixFunMinRes} query exposes system-level trade-offs and returns the corresponding Pareto-optimal implementation choices.
Case Study~II compares our approach with sequential optimization baselines.
Case Study~III illustrates how additional structural constraints (here, a homogeneity constraint) can be enforced and optimized over.
Finally, Case Study~IV shows how the framework supports co-design under uncertainty by jointly optimizing across multiple potential deployment scenarios.

Throughout, we report trade-offs over three global resources: \R{total cost}, \R{total energy}, and \R{mission completion time} \rev{(flipped to a resource with \Cref{def:fun_flip}).
For readability, we visualize these trade-offs through pairwise two-dimensional projections and show the nondominated Pareto front in each displayed resource plane.
Each plotted point represents a feasible multi-robot system together with a concrete selection of robot modules, fleet multiplicities, and planner.
}

\begin{remark}[\rev{Interpretation of pairwise Pareto projections}]
\label{rem:pareto_projections}
\rev{
For each panel, we project the three-dimensional Pareto-optimal resource vectors onto the displayed resource pair by discarding the omitted coordinate.
We then remove any projected points that become dominated with respect to the two displayed resources. 
Consequently, every point shown corresponds to a design that is Pareto-optimal in the full three-dimensional resource space, although some three-dimensional Pareto-optimal designs may be absent from a particular panel because their advantage lies in the omitted resource.
}
\end{remark}

\subsection{Case Study~I (Design Trade-offs \rev{and Phase Diagram})}
\label{sec:trade-off-phase-diagram}
We examine design trade-offs for a system tasked with covering a $400\,\mathrm{m} \times 800\,\mathrm{m}$ workspace.
The fleet may include up to three robots per type (maximum fleet size of nine).
We solve a \texttt{FixFunMinRes} query with required functionality fixed to~$m_c \geq 0.95$ (i.e., at least $95\%$ coverage), and the co-design solver minimizes the required system-level resources subject to this constraint.

\Cref{fig:pareto_front_case_study_1} shows three pairwise projections of the resulting Pareto front.

\begin{figure}[tb]
\centering
\setlength{\abovecaptionskip}{0pt}
\includegraphics[width=\linewidth]{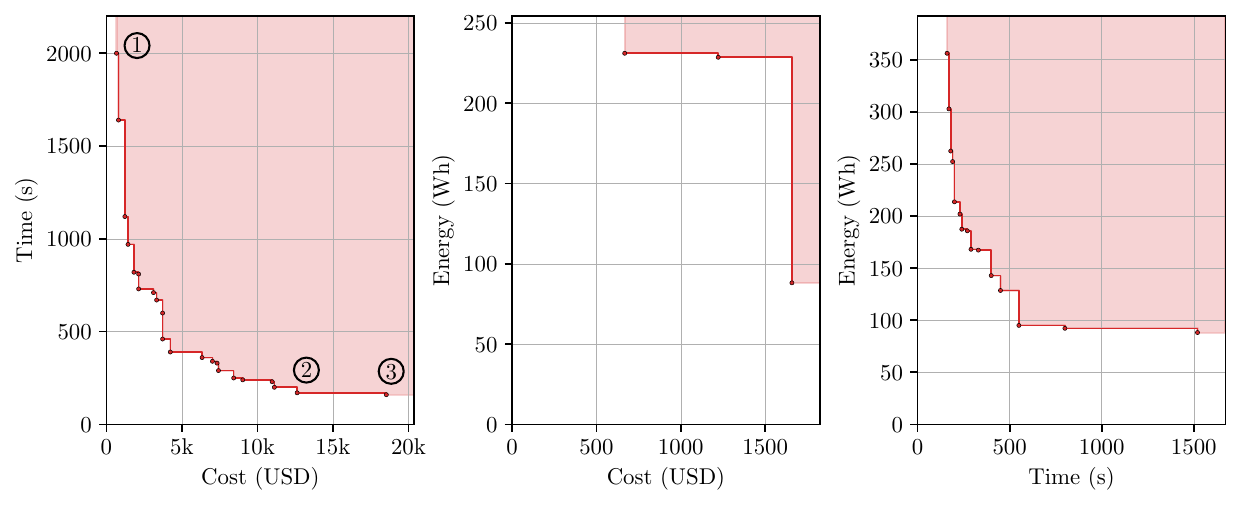}
\caption{%
Pairwise two-dimensional projections of the Pareto front for Case Study~I.
Left: \R{mission completion time} versus \R{total cost}.
Middle: \R{total energy} versus \R{total cost}.
Right: \R{total energy} versus \R{mission completion time}.
}
\label{fig:pareto_front_case_study_1}
\end{figure}

We highlight three points in the completion-time versus cost projection, which captures the speed--cost trade-off.
Design~\Circle{1} minimizes cost: a single medium-sized aerial robot equipped with the lower-cost actuation module~V1 and a 231.10\,Wh NiMH battery, coordinated by \gls{acr:darp}.
This configuration completes the mission in 33\,min\,20\,s at a cost of 667.77\,USD.
At the opposite extreme, design~\Circle{3} minimizes completion time.
It deploys two ground robots (\rev{15.33\,Wh} NiH$_2$ batteries), three medium-sized aerial robots (\rev{22.32\,Wh} NiH$_2$ batteries), and three large aerial robots (\rev{90.97\,Wh} NiH$_2$ batteries), with the higher-performance actuation module available (V2 for the ground and medium aerial robots, V1 for the large aerial robots).
This fleet uses the \gls{acr:agd} planner and finishes in 160\,s, requiring \rev{18,535.29\,USD} and \rev{370.51\,Wh}.

Design~\Circle{2} illustrates the value of joint reasoning across design layers:
three large aerial robots with \rev{100.92\,Wh} batteries and V1 actuation, coordinated by \gls{acr:darp}, complete the mission in 170\,s at \rev{12,628.83\,USD}, only 10\,s slower than design~\Circle{3} while saving \rev{5,906.46\,USD}.
Such a non-obvious trade-off would be difficult to identify without simultaneously optimizing robot modules, fleet multiplicities, and planner choice.

\rev{
Motivated by the phase-diagram perspective in \Cref{sec:introduction}, we examine how the optimal system architecture varies with the coverage requirement.
We sweep the requirement across $m_c \in [0.50, 0.95]$ in increments of $0.05$, while retaining the setup of Case Study~I.
For each value, we compute the three-dimensional Pareto front over \R{total cost}, \R{total energy}, and \R{mission completion time}.
We then impose a budget on one resource, minimize another, and classify the selected implementation by fleet composition, planner, or battery technology.
\Cref{fig:phase_diagrams_case_study_1} shows the resulting phase diagrams. 
The boundaries mark qualitative changes in the resource-minimizing architecture, and gray regions are infeasible.
}

\begin{figure}
\centering
\includegraphics[width=\linewidth]
{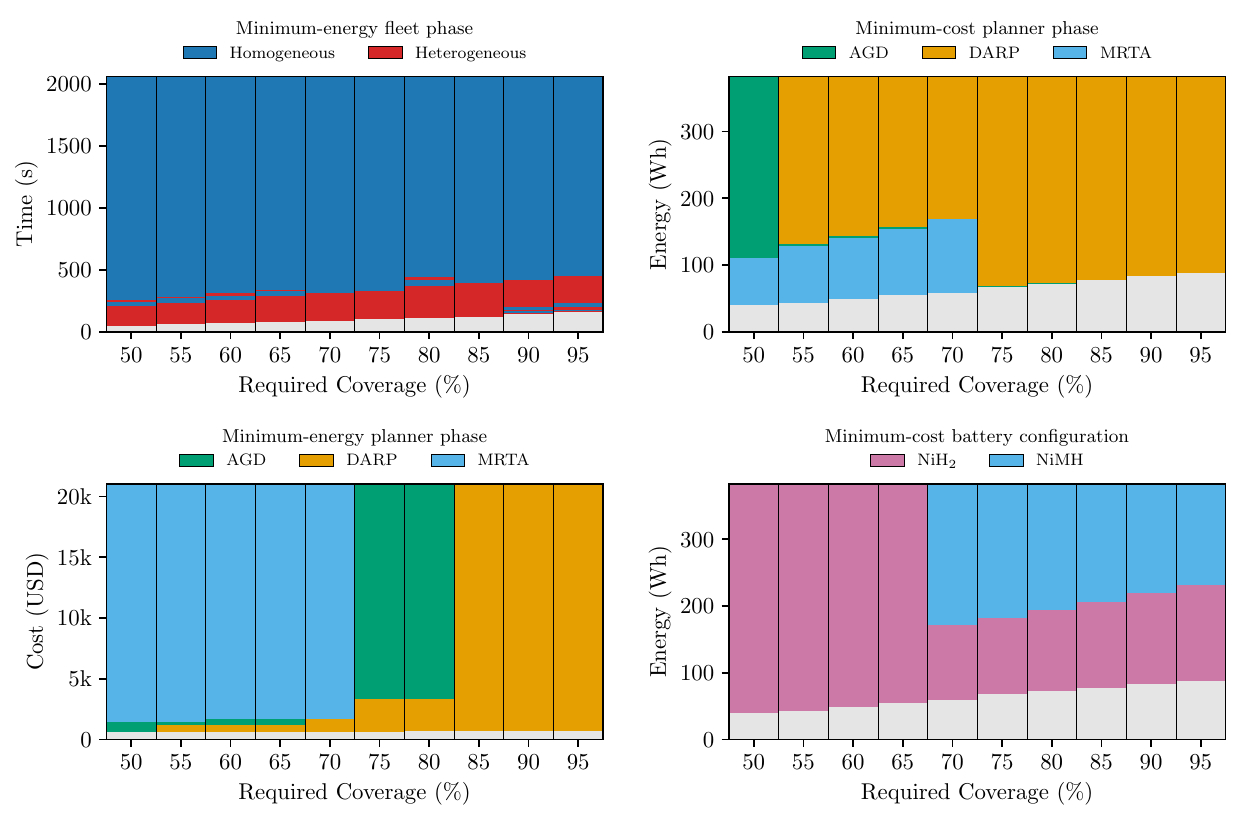}
\caption{
\rev{
Phase diagrams for Case Study~I as the required coverage varies.
Each colored region identifies the resource-minimizing implementation under the budget on the vertical axis; gray regions are infeasible.
Top left: fleet composition minimizing energy under a completion-time budget.
Top right: planner minimizing cost under an energy budget.
Bottom left: planner minimizing energy under a cost budget.
Bottom right: fleet-wide battery technology minimizing cost under an energy budget.
}
}
\label{fig:phase_diagrams_case_study_1}
\end{figure}

\rev{
The phase diagrams reveal clear architectural shifts as coverage increases.
The minimum feasible completion time rises from $50\,\mathrm{s}$ at $50\%$ coverage to $160\,\mathrm{s}$ at $95\%$. 
Tight time budgets favor heterogeneous fleets, whereas relaxed budgets generally favor lower-energy homogeneous fleets.
Under an energy budget, the cost-minimizing planner shifts from \gls{acr:mrta} and \gls{acr:agd} at low coverage to predominantly \gls{acr:darp} from $75\%$ onward.
Under a cost budget, the energy-minimizing planner is mostly \gls{acr:mrta} at $50\%$--$70\%$ coverage, \gls{acr:agd} at $75\%$--$80\%$, and \gls{acr:darp} at $85\%$--$95\%$.
Finally, of the eight available battery technologies, only NiH$_2$ and NiMH are selected under cost minimization.
NiMH first appears near $70\%$ coverage as its higher specific energy becomes advantageous for more demanding aerial designs.
Together, these transitions show how task requirements and resource objectives jointly determine the optimal architecture.
These diagrams are two-dimensional slices of a higher-dimensional design space whose phases could also be classified by fleet size, robot type, actuation modules, or other architectural properties.
The same procedure could vary additional task or environment parameters, with two-parameter sweeps directly partitioning the corresponding two-dimensional space.
}

Overall, Case Study~I shows that the framework produces interpretable Pareto fronts together with the implementation choices to realize each point, enabling principled exploration of system-level trade-offs \rev{and a phase diagram for heterogeneous fleet-algorithm design}.

\subsection{Case Study~II (Co-Design Vs. Sequential Optimization)}

\rev{
In engineering practice, system design decisions are often made sequentially based on prior insights. 
A planner may be selected before optimizing hardware and fleet composition, or robot designs may be fixed before choosing a planner. 
In this case study, we compare joint co-design with these two sequential workflows.
For both baselines, we assume a designer seeking to minimize completion time under a high budget limit (USD~8,000).
In the planner-first case, the designer selects \gls{acr:agd} because its trajectories have fewer turns and appear well suited to fast \glspl{acr:uav}.
The hardware and fleet composition are then optimized.
In the hardware-first case, the designer selects the highest-velocity actuation variant for each robot type: V1 for the large aerial robot and V2 for the medium aerial and ground robots. 
The planner and fleet composition are then optimized.
This comparison uses the same solver and is intended to illustrate and quantify the loss caused by fixing interacting subsystems, rather than to compare different optimization algorithms.
}

\rev{
\Cref{fig:pareto_front_case_study_2} compares the Pareto fronts obtained by these baselines against the full co-design solution from the previous case study.
After fixing \gls{acr:agd}, the best sequentially obtained design within this budget costs USD~7,928.14, completes the mission in 310~s, and requires 295.47~Wh.
Joint co-design instead selects \gls{acr:darp} with two large aerial robots, resulting in a cost of USD~7,419.94, a mission completion time of 290~s, and an energy requirement of 209.41~Wh.
The jointly selected system is therefore USD~508.20 cheaper, 20~s faster, and requires 86.06~Wh less energy.
This example illustrates how selecting a planner based on an isolated criterion can overlook its coupling with fleet composition, robot capabilities, and energy requirements.
}
\rev{
Across the complete fronts, each sequentially obtained design is either retained on the joint co-design front or dominated by a jointly optimized alternative.
Shared points correspond to cases in which the prior subsystem choice does not exclude a system-level optimum, whereas the remaining points quantify the opportunity cost of making that choice in isolation.
The two sequential fronts also cross in the \R{time}--\R{cost} and \R{energy}--\R{time} projections, showing that neither order of sequential decision-making is uniformly preferable.
}

For a quantitative comparison, we report the \gls{acr:hv}, \gls{acr:gdp}, and \gls{acr:igdp} indicators~\cite{li2019quality} for each front in \Cref{tab:indicators}; all fronts are normalized to $[0,1]$ before computation.
\gls{acr:hv} measures the volume between a dominated reference point and a given front and is maximized by the true Pareto front.
\gls{acr:gdp} and \gls{acr:igdp} are complementary distance-based indicators.
\rev{
\gls{acr:gdp} averages the minimum modified Euclidean distance from each candidate point to the reference front, while \gls{acr:igdp} averages this distance from each reference point to the candidate front.
Both indicators vanish when the fronts coincide, and lower values indicate smaller discrepancies under their respective distance measures.
}
We use the co-design Pareto front as the reference front for \gls{acr:gdp} and \gls{acr:igdp}, since it is guaranteed to dominate the sequential baselines by construction. For \gls{acr:hv}, we use the reference point~$(1,1,1)$, corresponding to the worst normalized value in each objective.
\rev{The results show that joint optimization improves the \gls{acr:hv} by up to 14.4\% over the sequential baselines}.
Moreover, the degradation in \gls{acr:hv} and \gls{acr:igdp} is larger under the fixed-robots baseline, indicating that robot design choices have a greater impact on overall system performance than planner selection alone.

Finally, we report the computational effort required to obtain the results for this case study.
\rev{The planning and execution simulations required approximately 16 hours of wall-clock time and were parallelized across up to eight worker processes on an Apple M1 MacBook Pro with eight CPU cores and 16~GB of memory.
The subsequent co-design solve required approximately 30 minutes of wall-clock time confirming that the framework introduces modest overhead relative to the underlying simulations and was executed as a single solver process without explicit parallelization.}

\begin{figure}[tb]
\centering
\setlength{\abovecaptionskip}{0pt}
\includegraphics[width=\linewidth]{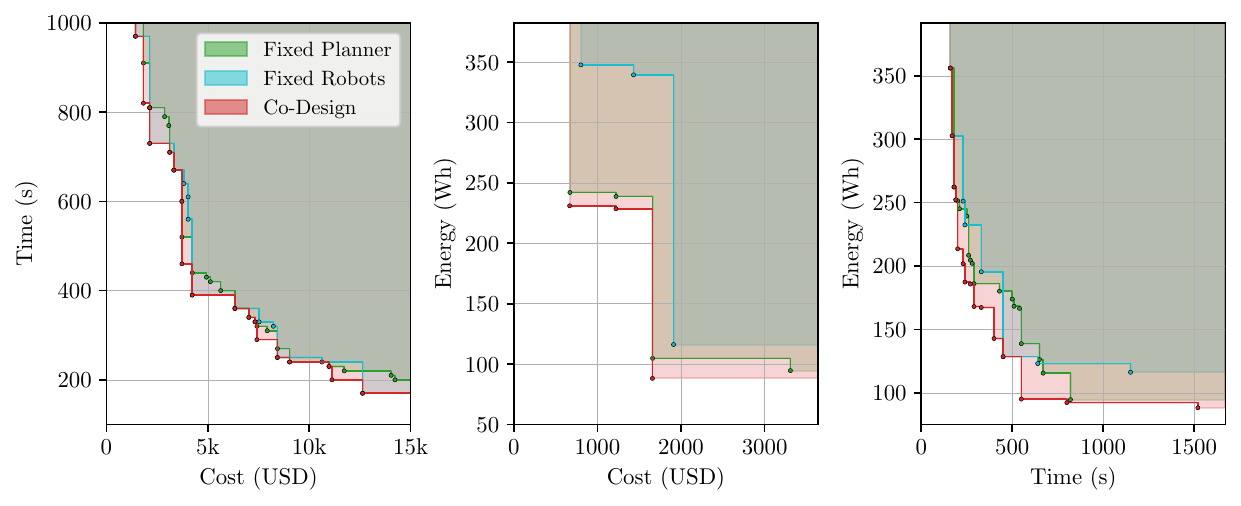}
\caption{
Pairwise two-dimensional projections of the Pareto fronts for Case Study~II, comparing co-design with sequential baselines.
Left: \R{mission completion time} versus \R{total cost}.
Middle: \R{total energy} versus \R{total cost}.
Right: \R{total energy} versus \R{mission completion time}.
}
\label{fig:pareto_front_case_study_2}
\end{figure}

\input{Figures/table_indicators}

\subsection{Case Study~III (Homogeneous vs. Heterogeneous)}
This case study demonstrates how additional structural constraints on the fleet can be expressed and optimized over within the same pipeline.
The task profile is identical to Case Study~I ($m_c \geq 0.95$ coverage of a $400\,\mathrm{m} \times 800\,\mathrm{m}$ workspace), and we compare: (i) the unconstrained setting permitting heterogeneous fleets (as before), and (ii) a constrained setting restricting the \textbf{Fleet Composer} to ground robots only (maximum fleet size of nine).

\Cref{fig:pareto_front_case_study_3} reports two projections of the resulting Pareto fronts.
In the homogeneous setting, all Pareto-optimal designs select \gls{acr:darp}, consistent with its original design objective for homogeneous coverage planning.
The dominance patterns differ across resource planes.
\rev{In the cost--time projection, allowing heterogeneous fleets strictly improves on the ground-only setting: for every Pareto-optimal ground-only design, there exists an alternative that is no slower and no more expensive, with a strict improvement in at least one of the two resources.}
\rev{The energy--time projection reveals a complementary trade-off. Over the completion-time range attainable by both fleet classes ($t \geq 170\,\mathrm{s}$), the homogeneous ground-fleet front lies on or below the heterogeneous front: every heterogeneous point in this range is weakly dominated by a homogeneous alternative. Heterogeneous fleets, however, achieve the globally shortest mission completion time of 160\,s, whereas the fastest homogeneous fleet requires 170\,s. This additional 10\,s reduction comes at a substantial energy cost: the corresponding heterogeneous design requires 356.15\,Wh, compared with 138.80\,Wh for the fastest homogeneous design.}
This can be explained by platform characteristics: ground robots are priced between the medium and large aerial robots but have substantially lower maximum velocities (longer mission times), while their smaller power coefficients and lower operating speeds yield significantly reduced energy consumption (cf.~\Cref{eq:instant_power_calculation}).

\begin{figure}[tb]
\centering
\setlength{\abovecaptionskip}{0pt}
\includegraphics[width=\linewidth]{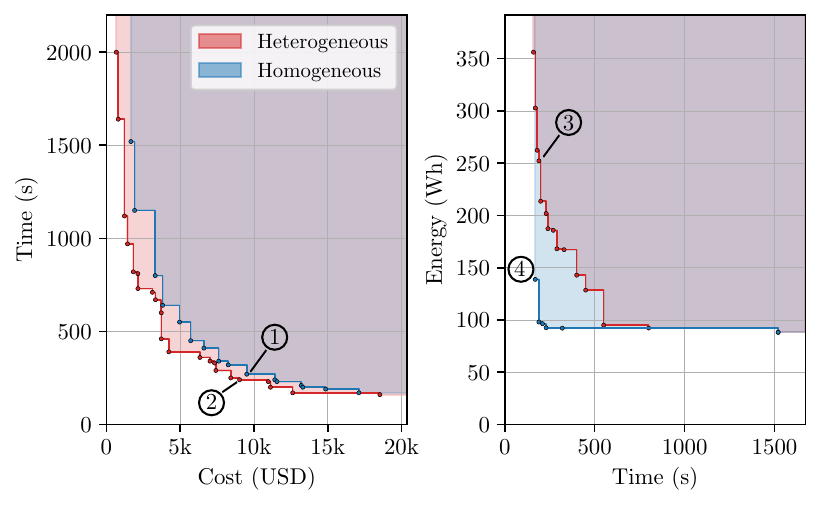}
\caption{%
Pairwise two-dimensional projections of the Pareto front for Case Study~III.
Left: \R{mission completion time} versus \R{total cost}.
Right: \R{total energy} versus \R{mission completion time}.}
\label{fig:pareto_front_case_study_3}
\end{figure}

Designs~\Circle{1} and~\Circle{2} illustrate the cost--time trade-off. 
Design~\Circle{1} is a homogeneous fleet of five ground robots with V2 actuation, coordinated by \gls{acr:darp}, completing the mission in 270\,s at \rev{9,512.26}\,USD. Design~\Circle{2} is a heterogeneous fleet, one medium-sized and two large aerial robots under the \gls{acr:agd} planner, that finishes in 240\,s at only 9,029.13\,USD. 
The heterogeneous fleet is thus simultaneously faster and cheaper.
The energy--time projection tells a different story. 
Design~\Circle{3}, a heterogeneous fleet of three ground robots\rev{, two medium-sized aerial robots} and three large aerial robots, completes the mission in \rev{190\,s} but consumes \rev{252.22\,Wh}. 
Design~\Circle{4}, a homogeneous fleet of nine ground robots, finishes in 170\,s while requiring only 138.80\,Wh, demonstrating that homogeneous ground fleets can be both faster and more energy-efficient when aerial platforms are excluded.

This case study demonstrates that the co-design framework allows structural constraints on multi-agent systems (such as fleet homogeneity) to be specified systematically and optimized over, yielding Pareto-optimal designs together with the concrete implementation choices that realize them while adhering to additional design constraints.

\subsection{Case Study~IV (Task Uncertainty)}
In this final case study, we demonstrate how the framework enables the design of multi-agent systems under uncertainty over deployment scenarios.
We consider two workspaces: $M_1$, a $400\,\mathrm{m} \times 800\,\mathrm{m}$ region, and $M_2$, a $500\,\mathrm{m} \times 640\,\mathrm{m}$ region containing five targets at unknown locations, modeled using the detection framework introduced in \Cref{subsec:coverage_detection_models}.
The system is optimized over three task profiles:
\begin{itemize}
    \item \textbf{Task~1:} $m_c \geq 0.95$ on $M_1$ (coverage only).
    \item \rev{\textbf{Task~2:} Achieve a probability of at least $0.95$ of detecting at least three of the five targets on $M_2$ (detection only).}
    \item \textbf{Task~3:} Satisfy both Task~1 and Task~2 simultaneously.
\end{itemize}
Tasks~1 and~2 represent scenarios in which the deployment condition is known a priori and the fleet is optimized for a single mission. 
Task~3, by contrast, models uncertainty over the deployment scenario: the fleet must be capable of fulfilling either mission, requiring designs that are robust across both task profiles.
\rev{
For Task~3, the reported mission completion time is the maximum of the completion times evaluated under the two map-specific task profiles.
}
We visualize the relation between the three tasks in \Cref{fig:cs3_tasks}.
The uncertain deployment setting is accommodated within our framework through the higher-order collections of waypoints, trajectories, and metrics introduced in \Cref{subsec:abstraction_details}.

\begin{figure}[tb]
\centering
\begin{tikzpicture}
\tikzset{edge/.style={draw=gray, thick}}
    \matrix (A) [matrix of nodes, row sep=0.5cm, column sep=0.3cm]
    {%
    & \textbf{Task 3} & \\
    \textbf{Task 1} & & \textbf{Task 2} \\
    };
    \draw[edge] (A-1-2)--(A-2-1);
    \draw[edge] (A-1-2)--(A-2-3);
\end{tikzpicture}
\caption{
Relation between the tasks in Case Study~IV.
A fleet that satisfies Task~3 is guaranteed to satisfy Task~1 and Task~2, while Task~1 and Task~2 are incomparable. 
}
\label{fig:cs3_tasks}
\end{figure}
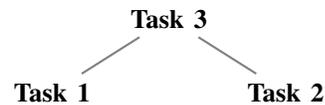

\Cref{fig:pareto_front_case_study_4} shows the resulting Pareto fronts for Case Study~IV. 
The figure consists of two pairwise projections: the left plot shows \R{mission completion time} versus \R{total cost}, while the right plot shows \R{total energy} versus \R{mission completion time}.
Before analyzing individual implementations, we first examine the overall structure of the Pareto fronts.
\rev{In both the full three-dimensional resource space and the displayed pairwise projections, the Pareto front for Task~2 resource-wise dominates that for Task~1: every point on the Task~1 front is dominated by at least one point on the Task~2 front. This indicates that, in the considered instantiation, achieving $95\%$ coverage of $M_1$ is more resource-intensive than satisfying the detection requirement on $M_2$. The Task~3 front partly coincides with that of Task~1, but also contains points that occur on neither single-task front, such as design~\Circle{2}.}
These solutions are Pareto-optimal \emph{only} under task uncertainty and therefore would not emerge from optimizing for either task in isolation.

\begin{figure}[tb]
\centering
\setlength{\abovecaptionskip}{0pt}
\includegraphics[width=\linewidth]{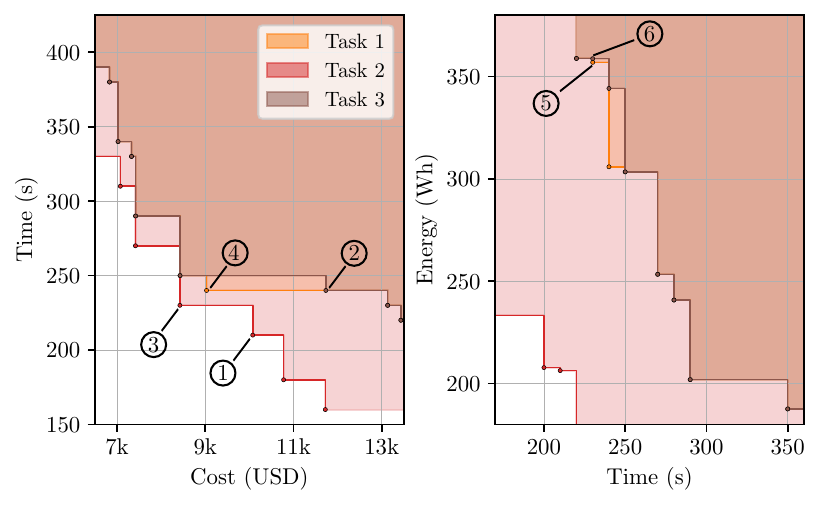}
\caption{%
Pairwise two-dimensional projections of the Pareto front for Case Study~IV.
Left: \R{mission completion time} versus \R{total cost}.
Right: \R{total energy} versus \R{mission completion time}.
}
\label{fig:pareto_front_case_study_4}
\end{figure}

Beyond these structural properties, changes in the task profile can fundamentally alter the optimal \emph{planner} selection, even when budgets are comparable.
Designs~\Circle{3} and~\Circle{4} illustrate this effect.
For Task~1, a budget of \$9029.13 yields a heterogeneous fleet of one medium and two large aerial robots under the \gls{acr:agd} planner (240\,s).
For Task~2, a comparable budget of \rev{\$8424.64} produces a different solution: two large aerial robots under the \gls{acr:darp} planner (230\,s).
The shift to \gls{acr:darp} is noteworthy: although traditionally associated with homogeneous coverage, its denser visitation patterns--induced by its use of the minimum sensing radius in heterogeneous fleet settings--increase cumulative exposure (cf.~\Cref{eq:exposure,eq:fleet_hit_probability}), making it advantageous for probabilistic detection even when it is suboptimal for strict coverage efficiency.
We also observe instances where the \emph{same} fleet composition remains Pareto-optimal across Task~1 and Task~2 while the co-design solution switches the planner between \gls{acr:agd} (coverage) and \gls{acr:darp} (detection), highlighting that planner choice is itself a first-class design variable whose optimality can depend strongly on task structure.

\rev{Finally, designs~\Circle{5} and~\Circle{6} show that task uncertainty can alter the internal robot designs even when the optimal fleet composition remains unchanged.
Both designs use two robots of each type, actuation variants $(V1,V2,V2)$, NiH$_2$ batteries, and \gls{acr:agd}.
For Task~1, design~\Circle{5} completes the mission in 230\,s and requires 356.87\,Wh. Design~\Circle{6}, which must support both Task~1 and Task~2, retains the same fleet composition and completion time but requires 358.66\,Wh.
Although Task~1 is generally more resource-intensive in aggregate, the trajectories for Task~2 place slightly higher motion-energy requirements on the medium aerial and ground robots.
Consequently, their battery capacities increase from 26.57\,Wh to 27.29\,Wh and from 21.21\,Wh to 21.38\,Wh, respectively.
This illustrates that selecting a fleet based only on the apparently most demanding task is insufficient: different tasks can constrain different internal components, which must be considered jointly when designing for uncertain deployment conditions.}

Overall, Case Study~IV demonstrates that the framework naturally extends to multi-task settings and uncertain deployment conditions.
By jointly optimizing robot design, fleet composition, and planning across tasks, it identifies Pareto-optimal designs that would not emerge from single-task optimization and exposes non-obvious interactions between planner structure, fleet heterogeneity, and task metrics.

\subsection{\rev{Computational Scaling and Black-Box Baselines}}
\label{sec:computation-baselines}

\rev{
Finally, we empirically illustrate the discussion of \Cref{subsec:computational_co_design} by solving Case Study~I on nested subsets of the simulated trajectories underlying the full case study, from $25\%$ to $100\%$ in steps of $25\%$.
The full set consists of the 945 trajectories that reach the required $95\%$ coverage among the $3\times342$ planner runs over all fleets with at most three robots per type.
We report antichain widths at three levels of the co-design model: 
$w_{\mathrm{plan}}$, the width of the antichain of task-feasible planning and execution outputs; 
$w_{\mathrm{joint}}$, the width of the intermediate antichain obtained after composing these outputs with the fleet design; 
and $w_{\mathrm{out}}$, the width of the final Pareto front.
We also report the number of fixed-point iterations $K$ required for convergence.
Because solver optimization restructures the composite problem before solving it, these quantities do not correspond one-to-one to the interfaces of the formal model. 
They are nevertheless representative of the pruning discussed in \Cref{subsec:computational_co_design}.
}

\begin{table}[tb]
\setlength{\abovecaptionskip}{1pt}
\centering
\footnotesize
\caption{
\rev{Observed antichain widths for nested, planner-balanced subsets of the Case Study~I planning catalog.}
}
\begin{NiceTabular}{cc|cccc}
    \toprule
    \rev{Subset} & \rev{Entries} & \rev{$w_{\mathrm{plan}}$} & \rev{$w_{\mathrm{sys}}$} & \rev{$w_{\mathrm{out}}$} & \rev{$K$}\\
    \midrule
    \rev{$25\%$}  & \rev{237} & \rev{194} & \rev{58}  & \rev{51} & \rev{$3$} \\
    \rev{$50\%$}  & \rev{473} & \rev{351} & \rev{176} & \rev{52} & \rev{$4$} \\
    \rev{$75\%$}  & \rev{710} & \rev{496} & \rev{79}  & \rev{72} & \rev{$4$} \\
    \rev{$100\%$} & \rev{945} & \rev{616} & \rev{83}  & \rev{75} & \rev{$4$} \\
    \bottomrule
\end{NiceTabular}
\label{tab:cs1_mcdp_scaling}
\end{table}

\rev{
From \Cref{tab:cs1_mcdp_scaling}, we observe that $w_{\mathrm{plan}}$ grows approximately with the number of catalog entries. 
At this stage, one item dominates another only when the same fleet or a sub-fleet requires no more execution time or energy without worsening mission-completion time, a condition rarely satisfied by independently generated plans. 
By contrast, $w_{\mathrm{sys}}$ varies nonmonotonically, dropping from 176 to 79 between the $50\%$ and $75\%$ subsets. 
Composing the planning requirements with the fleet design can make previously incomparable alternatives comparable in terms of total cost and energy.
Consequently, $w_{\mathrm{sys}}$ depends on the partial-order structure rather than on catalog cardinality alone, and new entries may dominate more partial designs than they introduce.
Finally, the fixed-point computation requires only three or four iterations for all subsets, showing that on this instance, the iteration depth remains stable as the catalog grows.
}

\begin{figure}[tb]
\centering
\setlength{\abovecaptionskip}{0pt}
\includegraphics[width=0.95\linewidth]{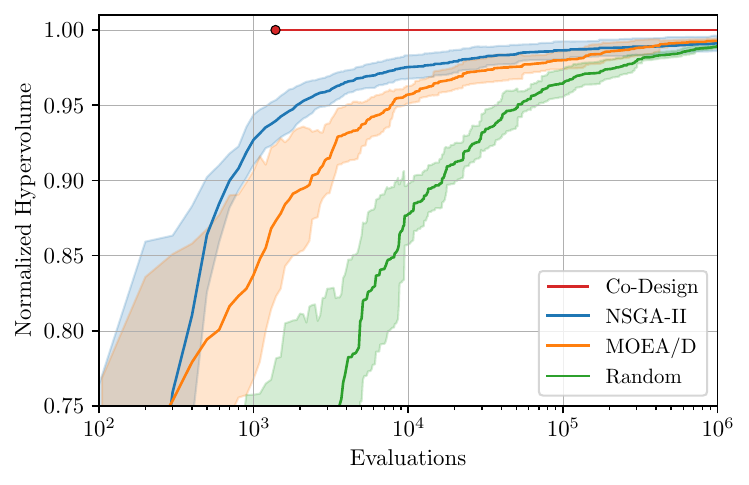}
\caption{%
\rev{
Hypervolume convergence of the black-box baselines. 
The horizontal line starting point marks equal runtime between co-design and \NSGAII.
Each curve is an average of 10 independent runs, and the shaded regions represent one standard deviation.
}
}
\label{fig:convergence_plot}
\end{figure}

\rev{
Having examined the computational scaling of co-design on Case Study~I, we next evaluate how quickly black-box optimizers approach the optimal co-design front on a variant of the same case study.
We compare co-design against three black-box multi-objective baselines:
\MOEAD~\cite{zhang2007moea}, which decomposes the multi-objective problem into cooperatively optimized scalar subproblems;
\NSGAII~\cite{deb2002fast}, which uses nondominated sorting and crowding distance to balance convergence and diversity;
and brute-force uniform random search over the implementation space.
}

\rev{
\Cref{fig:convergence_plot} shows the normalized \gls{acr:hv} achieved by each baseline as a function of the number of design evaluations.
The horizontal line starting point marks the \NSGAII evaluation count at which its cumulative runtime equals the time required by the co-design solver to compute the optimal solution.
The evolutionary baselines converge substantially faster than uniform random search, demonstrating the benefit of guided exploration. 
Nevertheless, at the equal-runtime point, their hypervolume gaps remain approximately $5\%$--$15\%$, compared with approximately $35\%$ for random search. 
A measurable gap remains even after one million evaluations. 
Unlike the black-box methods, co-design exploits the order structure of the problem to prune dominated alternatives and compute the complete optimal front.
The stochastic baselines can approximate this front but cannot certify its completeness or global Pareto optimality.
}

%% file: Figures/table_battery_modules.tex
\begin{table}[tb]
    \setlength{\abovecaptionskip}{1pt}
    \centering
    \caption{Available battery modules.}
    \renewcommand{\arraystretch}{1}
    \label{tab:battery-modules}
    \begin{NiceTabular}{c|cc}
        \toprule
        \textbf{Technology} & $\boldsymbol{\rho}$ (Wh/kg) & $\boldsymbol{\alpha}$ (Wh/USD) \\
        \midrule
        LiPo   & 250 & 2.50 \\
        LCO    & 195 & 2.84 \\
        LMO    & 150 & 2.84 \\
        NiH$_2$ & 45  & 10.5 \\
        NiMH   & 100 & 3.41 \\
        LFP    & 90  & 1.50 \\
        NiCad  & 30  & 0.50 \\
        SLA    & 30  & 7.0 \\
        \bottomrule
    \end{NiceTabular}
\end{table}

%% file: Figures/table_actuation_modules.tex
\begin{table}[tb]
    \setlength{\abovecaptionskip}{1pt}
    \centering
    \caption{Available actuation modules.}
    \renewcommand{\arraystretch}{1}
    \label{tab:actuation-modules}
    \resizebox{\columnwidth}{!}{%
    \begin{NiceTabular}{c|cccccccccccc}
        \toprule
        \Block{2-1}{\textbf{Robot}} & 
        \Block{2-1}{\textbf{Variant}} &
        $\boldsymbol{v_{\max}}$ &
        $\boldsymbol{a_{\max,\mathrm{lat}}}$ &
        $\boldsymbol{a_{\max,\mathrm{long}}}$ &
        $\boldsymbol{r_{\text{turn}}}$ &
        $\boldsymbol{\dot{\theta}_{\max}}$ &
        $\boldsymbol{c_{\text{vel}}}$ &
        $\boldsymbol{c_{\text{acc}}}$ &
        $\boldsymbol{m_{\text{max}}}$ &
        $\boldsymbol{P_{\text{idle}}}$ &
        \textbf{Cost} &
        \textbf{Mass} \\
        & &
        (m/s) & (m/s$^{2}$) & (m/s$^{2}$) & (m) & (rad/s) &
        (W\,s$^{2}$/m$^{2}$) & (W\,s$^{2}$/m) & kg & (W) & (USD) & kg \\
        \midrule
        \multirow{2}{*}{Aerial (M)} & V1 & 10.0 & 1.5 & 2.0 & 1.0 & 1.0 & 4 & 7 &  5 & 50 &  500 & 1 \\
         & V2 & 13.0 & 1.5 & 2.0 & 1.0 & 1.0 & 5 & 9 &  7 & 70 &  600 & 1.5\\
        \hline
        \multirow{2}{*}{Aerial (L)} & V1 & 20.0 & 2.0 & 4.0 & 3 & 1.5 & 10 & 20 & 15 & 250 & 4000 & 4.0 \\
         & V2 & 16.0 & 1.5 & 3.0 & 3 & 1.5 & 8 & 14 & 10 & 200 & 3500 & 3.5 \\
        \hline
        \multirow{2}{*}{Ground} & V1 &  5.0 & 1.5 & 1.0 & 1.0 & 1.0 & 1 & 3 & 100 & 150 & 1500 & 10\\
         & V2 &  7.0 & 1.5 & 1.0 & 1.0 & 1.0 & 3 & 5 & 100 & 200 & 1750 & 12\\
        \bottomrule
    \end{NiceTabular}%
    }
\end{table}

%% file: Figures/table_sensing_modules.tex
\begin{table}[tb]
    \setlength{\abovecaptionskip}{1pt}
    \centering
    \renewcommand{\arraystretch}{1.0}
    \caption{Available sensing modules.}
    \label{tab:sensing-modules}
    \resizebox{\columnwidth}{!}{
    \begin{NiceTabular}{c|cccccccc}
        \toprule
        \textbf{Robot} & \textbf{Variant} &
        $\boldsymbol{r_{\text{s}}}$ (m) &
        $\lambda_{\mathrm{base}}$ &
        $\sigma_d$ (m) &
        $\beta_v$ (s/m) &
        $\boldsymbol{P_{\text{req}}}$ (W) &
        \textbf{Cost} (USD) & \textbf{Mass} (kg) \\
        \midrule
        Aerial (M) & imp1 & 12 & 0.95 & 8 & 0.004 & 20 & 100 & 0.4\\
        \hline
        Aerial (L) & imp1 & 40 & 0.98 & 30 & 0.02 & 68 & 200 & 1.0 \\
        \hline
        Ground & imp1 & 30 & 0.99 & 26 & 0.02 & 35 & 150 & 2.5\\
        \bottomrule
    \end{NiceTabular}
    }
\end{table}

%% file: Figures/table_indicators.tex
\begin{table}[tb]
\setlength{\abovecaptionskip}{1pt}
\centering
\footnotesize
\caption{\rev{Quality indicators for the co-design and sequential baselines.}}
\begin{adjustbox}{max width=\columnwidth, center}
\begin{NiceTabular}{c|ccc}
    \toprule
    \textbf{Metric} & \textbf{Co-Design} & Fixed Planner &  Fixed Robots\\
    \midrule
    \gls{acr:hv} ($\uparrow$)     & $\rev{\mathbf{0.8484}}$ & $\rev{0.8028}$ & $\rev{0.7419}$\\
    \gls{acr:gdp} ($\downarrow$)  & $\rev{\mathbf{0.0}}$    & $\rev{0.0143}$ & $\rev{0.0181}$\\
    \gls{acr:igdp} ($\downarrow$) & $\rev{\mathbf{0.0}}$    & $\rev{0.0132}$ & $\rev{0.0253}$\\
    \bottomrule
\end{NiceTabular}
\end{adjustbox}
\label{tab:indicators}
\end{table}

%% file: Sections/7_conclusion_future_work.tex
\section{Conclusion and Future Work}
\label{sec:conclusion_future_work}

In this paper, we introduced a formal, compositional, and task-aware framework to co-design the hardware and algorithms of heterogeneous multi-agent robotic systems.
General, solution-agnostic abstractions of robots, fleets, planners, executors, and evaluators are provided, and domain experts are free to explore new designs that fit into them, without worrying about integration with other components.
We then derive co-design models based on these abstractions, and interconnect them to globally optimize the entire system for various types of tasks.
Modularity of co-design allows for more sophisticated models by further decomposing the components into sub-components.
The compositional solver provided by co-design efficiently explores the combinatorial, large design space, without exhaustively simulating all combinations.

\rev{The simulation-based case studies provide a controlled proof of concept of the compositional methodology and its ability to expose system-level trade-offs, rather than comprehensive validation on physical heterogeneous robot fleets.
The framework itself is not simulation-specific: component models may be populated analytically, numerically, or from simulations, experimental measurements, system-identification data, and physical evaluations.
Physical validation remains an important direction for future work.
The framework is limited to models whose exposed functionality--resource interfaces admit a semantics-preserving monotone formulation; systems for which no such abstraction exists fall outside its scope.}

As a first step in a general, compositional, collaborative, and interpretable task-aware co-design framework for heterogeneous robot fleets, this paper paves the way to universal \emph{robotics phase diagrams}.
There are many other aspects to include in the framework, with a few named here.
We assume a centralized planner for waypoint generation and local controllers for tracking and execution. 
Incorporating communication capabilities as design variables and studying partially distributed algorithms and their trade-offs is an important direction for future work.
Although the abstractions and co-design problems are general, small collections of robot types, tasks, and planners are considered in this paper.
Exploring larger catalogs will get us closer to the final robot phase diagram.
Recent work \cite{huang2026distributional} introduces distributional uncertainty and adaptive decision processes into co-design.
With risk and uncertainty as important factors in real designs, incorporating them is a crucial step.
Finally, some designs are expensive to evaluate (involve experiments, computationally heavy simulations, etc.), calling for a performance-aware online learning method to cleverly decide new solutions to explore.